%% file: ArxivPreprint.tex
\documentclass[10pt,twocolumn,letterpaper]{article}

\usepackage[pagenumbers]{style} 

\makeatletter
\@namedef{ver@everyshi.sty}{}
\makeatother
\usepackage{tikz}

\usepackage[utf8]{inputenc} 
\usepackage[T1]{fontenc}    
\usepackage{url}            
\usepackage{booktabs}       
\usepackage{amsfonts}       
\usepackage{nicefrac}       
\usepackage{microtype}      
\usepackage{xcolor}         
\usepackage{svg}
\usepackage{amsmath} 
\usepackage{amssymb} 
\usepackage{amsthm} 
\usepackage{bm} 
\usepackage{tabularx}
\usepackage{wrapfig}
\usepackage{bbm}
\usepackage{enumitem}
\usepackage{multicol}
\usepackage{pifont}
\usepackage[accsupp]{axessibility}  

\usepackage[pagebackref,breaklinks,colorlinks]{hyperref}

\DeclareMathOperator*{\argmax}{arg\,max}

\newcommand\numberthis{\addtocounter{equation}{1}\tag{\theequation}}
\newcommand{\comment}[1]{}

\newtheorem{theorem}{Theorem}
\newtheorem{corollary}{Corollary}[theorem]

\usepackage[capitalize]{cleveref}
\crefname{section}{Sec.}{Secs.}
\Crefname{section}{Section}{Sections}
\Crefname{table}{Table}{Tables}
\crefname{table}{Tab.}{Tabs.}

\begin{document}


\title{Balanced Product of Calibrated Experts for Long-Tailed Recognition}

\author{
  Emanuel Sanchez Aimar$^{1}$,
  Arvi Jonnarth$^{1,}$\thanks{Affiliation: Husqvarna Group, Huskvarna, Sweden.}, 
  Michael Felsberg$^{1,}$\thanks{Co-affiliation: University of KwaZulu-Natal, Durban, South Africa.}, Marco Kuhlmann$^{2}$\\
  $^{1}$
  Department of Electrical Engineering, 
  Linköping University, Sweden \\
  $^{2}$
  Department of Computer and Information Science, 
  Linköping University, Sweden \\
  \tt\small{\{emanuel.sanchez.aimar,arvi.jonnarth,michael.felsberg,marco.kuhlmann\}@liu.se}
}

\maketitle

\begin{abstract}

\input{abstract}

\end{abstract}


\section{Introduction}

\input{introduction}

\section{Related work}

\input{related_work}

\section{Method}

\input{preliminaries}

\input{method}

\subsection{Meeting the calibration assumption}

\input{calibration}

\section{Experiments}

\input{experiments}

\section{Conclusion}

\input{conclusions}

\section{Acknowledgements}
\input{acknowledgments}

{\small
\bibliographystyle{ieee_fullname}
\bibliography{egbib}
}


\onecolumn
\clearpage
\newpage
\input{appendix}

\end{document}

%% file: abstract.tex
Many real-world recognition problems are characterized by long-tailed label distributions. These distributions make representation learning highly challenging due to limited generalization over the tail classes. If the test distribution differs from the training distribution, e.g.\ uniform versus long-tailed, the problem of the distribution shift needs to be addressed. A recent line of work proposes learning multiple diverse experts to tackle this issue. Ensemble diversity is encouraged by various techniques, e.g.\ by specializing different experts in the head and the tail classes. In this work, we take an analytical approach and extend the notion of logit adjustment to ensembles to form a Balanced Product of Experts (BalPoE). BalPoE combines a family of experts with different test-time target distributions, generalizing several previous approaches. We show how to properly define these distributions and combine the experts in order to achieve unbiased predictions, by proving that the ensemble is Fisher-consistent for minimizing the balanced error. Our theoretical analysis shows that our balanced ensemble requires calibrated experts, which we achieve in practice using mixup. We conduct extensive experiments and our method obtains new state-of-the-art results on three long-tailed datasets: CIFAR-100-LT, ImageNet-LT, and iNaturalist-2018. Our code is available at \url{https://github.com/emasa/BalPoE-CalibratedLT}. 

%% file: introduction.tex
\begin{figure}[t]
    
    \centering
    
	\resizebox{0.8\columnwidth}{!}{
	\setlength{\tabcolsep}{0pt}
	\begin{tabular}{cccc}
		\includegraphics[width=0.235\textwidth]{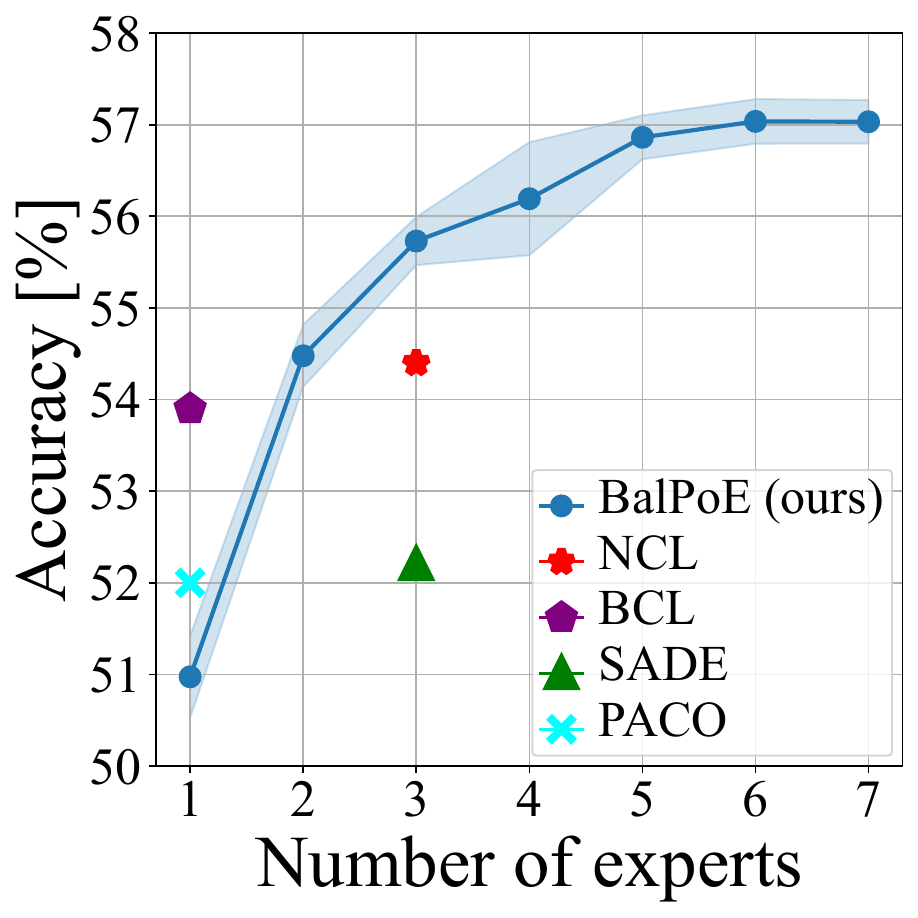} & 	
		\includegraphics[width=0.25\textwidth]{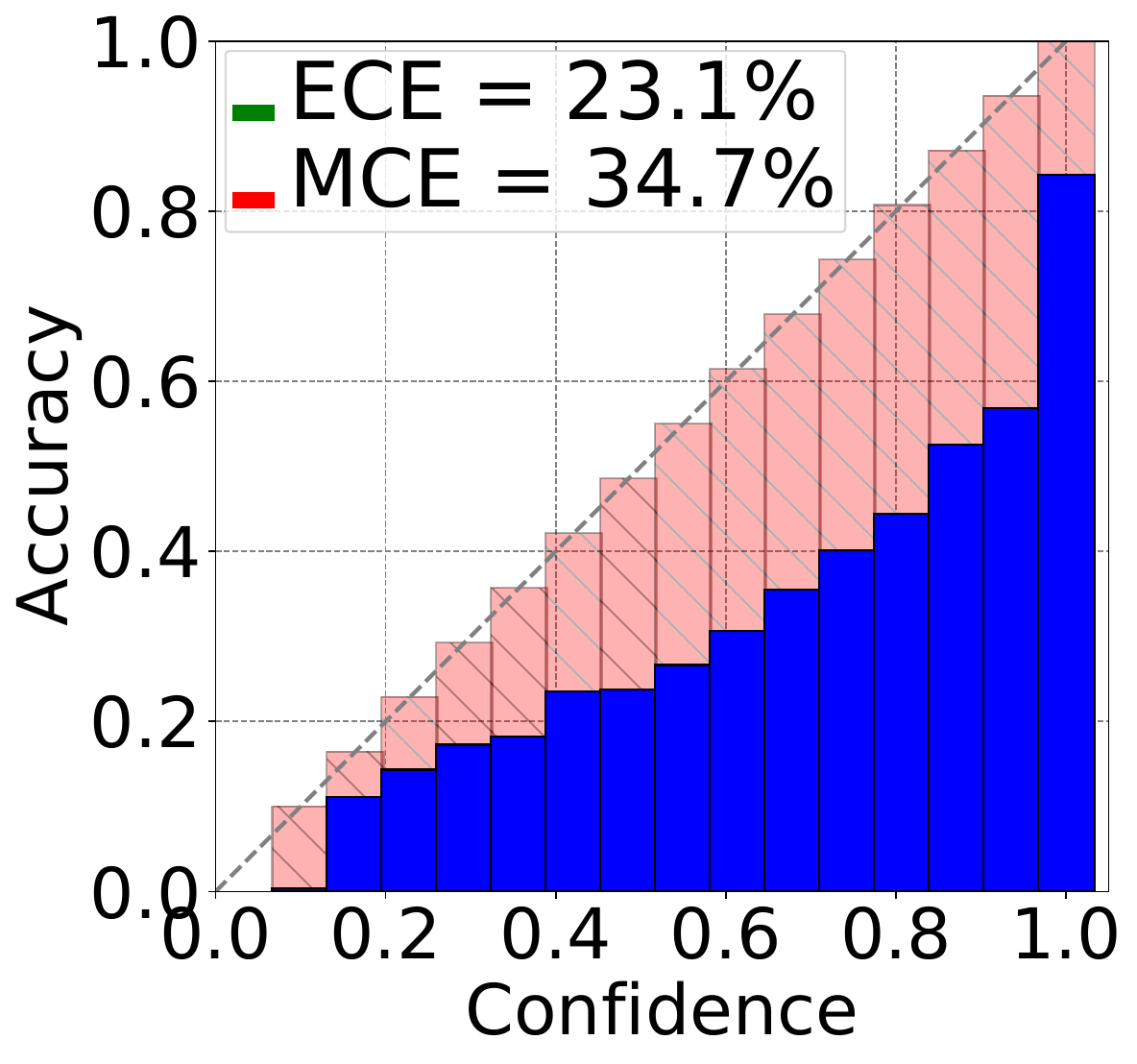} \\
		(a) SOTA comparison & (b) BS \\
		\includegraphics[width=0.25\textwidth]{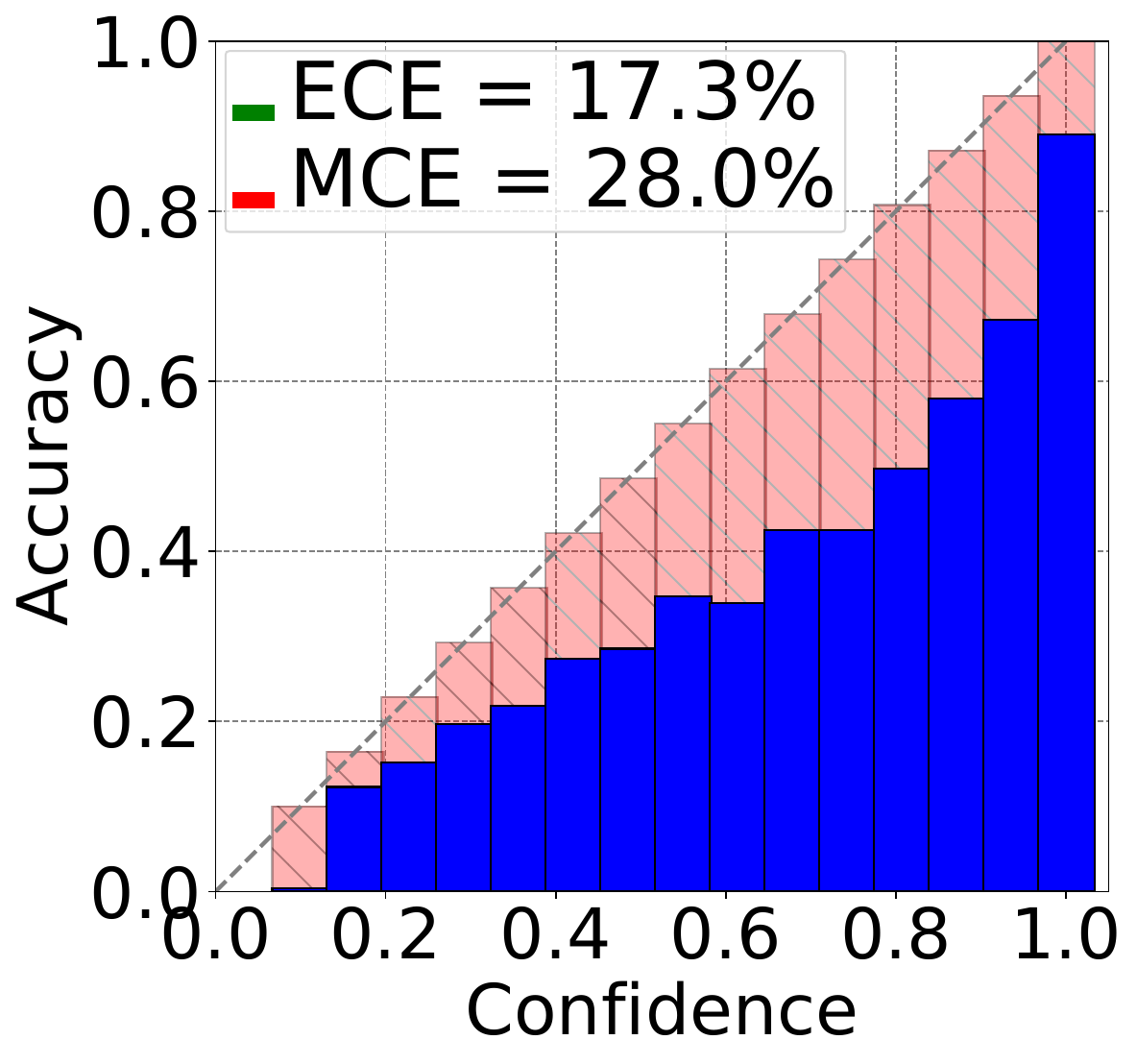} &		
		\includegraphics[width=0.25\textwidth]{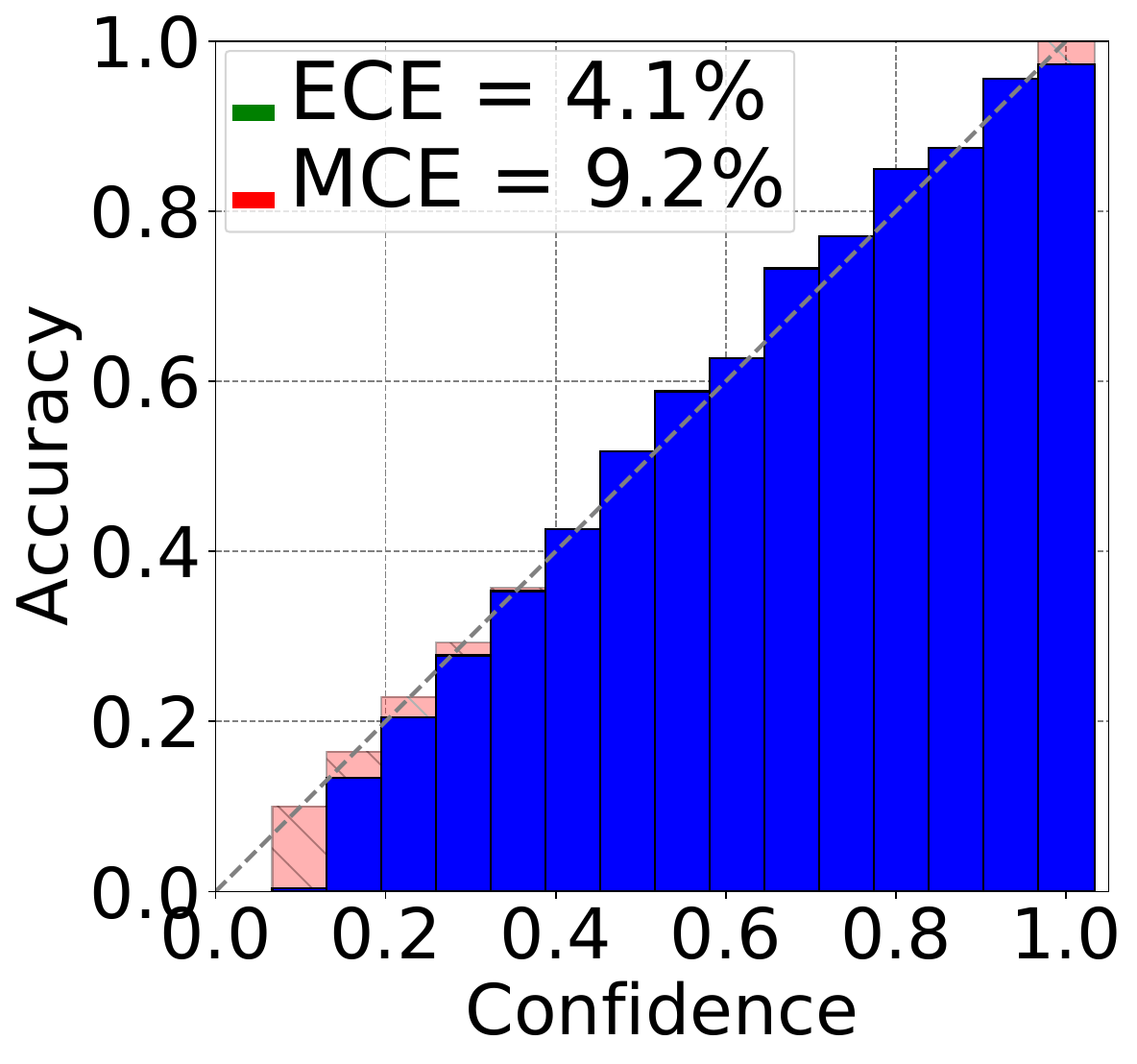} \\
		(c) SADE & (d) Our approach
	\end{tabular}
	
	}

    \vspace{-6pt}
 
	\caption{(a) SOTA comparison. Many SOTA approaches employ Balanced Softmax (BS) \cite{ren2020balanced} and its extensions to mitigate the long-tailed bias \cite{Cui2021PaCo,li2022NCL_nested_collab,zhu2022BLC_balanced_constrastive,zhang2021test}. However, we observe that single-expert bias-adjusted models, as well as multi-expert extensions, are still poorly calibrated by default. Reliability plots for (b) BS, (c) SADE \cite{zhang2021test} and (d) Calibrated BalPoE (our approach).}

        \vspace{-4pt}
 
	\label{fig:first_page}

        \vspace{-12pt}

\end{figure}

Recent developments within the field of deep learning, enabled by large-scale datasets and vast computational resources, have significantly contributed to the progress in many computer vision tasks \cite{krizhevsky2012imagenet}. However, there is a discrepancy between common evaluation protocols in benchmark datasets and the desired outcome for real-world problems. Many benchmark datasets assume a balanced label distribution with a sufficient number of samples for each class. In this setting, empirical risk minimization (ERM) has been widely adopted to solve multi-class classification and is the key to many state-of-the-art (SOTA) methods, see Figure \ref{fig:first_page}. Unfortunately, ERM is not well-suited for imbalanced or long-tailed (LT) datasets, in which \textit{head classes} have many more samples than \textit{tail classes}, and where an unbiased predictor is desired at test time. This is a common scenario for real-world problems, such as object detection \cite{lin2014coco}, medical diagnosis \cite{grzymala2004approach} and fraud detection \cite{philip1998toward}. On the one hand, extreme class imbalance biases the classifier towards head classes \cite{kang2019decoupling,tang2020causal_momentum}. On the other hand, the paucity of data prevents learning good representations for less-represented classes, especially in few-shot data regimes \cite{ye2020cdt}. Addressing class imbalance is also relevant from the perspective of algorithmic fairness, since incorporating unfair biases into the models can have life-changing consequences in real-world decision-making systems \cite{mehrabi2021survey}.

Previous work has approached the problem of class imbalance by different means, including \textit{data re-sampling} 
\cite{chawla2002oversampling-smote,buda2018systematic}, \textit{cost-sensitive learning} \cite{xie1989logit,akbani2004applying}, and \textit{margin modifications} \cite{cao2019learning,tan2020equalization}. While intuitive, these methods are not without limitations. Particularly, over-sampling can lead to overfitting of rare classes \cite{cui2019effective}, under-sampling common classes might hinder feature learning due to the omitting of valuable information \cite{kang2019decoupling}, and loss re-weighting can result in optimization instability \cite{cao2019learning}. Complementary to these approaches, ensemble learning has empirically shown benefits over single-expert models in terms of generalization \cite{kuncheva2014combining} and predictive uncertainty \cite{lakshminarayanan2017deep_ensembles} on balanced datasets. Recently, expert-based approaches \cite{wang2020experts,cai2021ace} also show promising performance gains in the long-tailed setting. However, the theoretical reasons for how expert diversity leads to better generalization over different label distributions remain largely unexplored.

In this work, we seek to formulate an ensemble where each expert targets different classes in the label distribution, and the ensemble as a whole is provably unbiased. We accomplish this by extending the theoretical background for logit adjustment to the case of learning diverse expert ensembles. We derive a constraint for the target distributions, defined in terms of expert-specific biases, and prove that fulfilling this constraint yields Fisher consistency with the \textit{balanced error}. In our formulation, we assume that the experts are calibrated, which we find not to be the case by default in practice. Thus, we need to assure that the assumption is met, which we realize using mixup \cite{zhang2017mixup}. 
Our contributions can be summarized as follows:
\begin{itemize}[leftmargin=*]
    \item We extend the notion of logit adjustment based on label frequencies to balanced ensembles. We show that our approach is theoretically sound by proving that it is \textit{Fisher-consistent} for minimizing the balanced error.
    \vspace{-2pt}
    \item Proper calibration is a necessary requirement to apply the previous theoretical result. We find that mixup is vital for expert calibration, which is not fulfilled by default in practice. Meeting the calibration assumption ensures Fisher consistency, and performance gains follow.
    \vspace{-2pt}
    \item Our method reaches new state-of-the-art results on three long-tailed benchmark datasets.
    \vspace{-2pt}
\end{itemize}

%% file: related_work.tex
\textbf{Data resampling and loss re-weighting.} Common approaches resort to resampling the data or re-weighting the losses to achieve a more balanced distribution. Resampling approaches include under-sampling the majority classes \cite{drummond2003undersampling}, over-sampling the minority classes \cite{chawla2002oversampling-smote,han2005borderline-oversampling}, and class-balanced sampling \cite{huang2016learning,wang2017learning}. Re-weighting approaches are commonly based on the inverse class frequency \cite{huang2016learning}, the \textit{effective per-class number of samples} \cite{cui2019effective}, or sample-level weights \cite{lin2017focal}. In contrast, we make use of theoretically-sound margin modifications without inhibiting feature learning.

\textbf{Margin modifications.} Enforcing large margins for minority classes has been shown to be an effective regularizer under class imbalance \cite{cao2019learning,tan2020equalization}. Analogously, a posthoc \textit{logit adjustment} (LA) can be seen as changing the class margins during inference time to favor tail classes \cite{menon2020adjustment,tang2020causal_momentum}. Particularly, the approach proposed by Menon \etal \cite{menon2020adjustment} is shown to be \textit{Fisher-consistent} \cite{Lin2002fisher_consistency} for minimizing the balanced error. Finally, Hong \etal \cite{hong2021lade} generalize LA to accommodate an arbitrary, known target label distribution. We extend single-model LA to a multi-expert framework by parameterizing diverse target distributions for different experts and show how to formulate the margins such as to maintain Fisher consistency for the whole ensemble.

\textbf{Calibration.} Modern over-parameterized neural networks are notorious for predicting uncalibrated probability estimates \cite{guo2017calibration}, being wrongly overconfident in the presence of out-of-distribution data \cite{li2020wrongly_overconfident}. These issues are further exacerbated under class imbalance \cite{zhong2021mislas_mixup}. Mixup \cite{zhang2017mixup} and its extensions \cite{verma2019manifold,yun2019cutmix} have been effective at improving confidence calibration \cite{thulasidasan2019mixup_calibration}, and to some extent, generalization \cite{zhang2017mixup,yun2019cutmix}, in the balanced setting. However, mixup does not change the label distribution \cite{xu2021bayias}, thus several methods modify sampling \cite{xu2021bayias,park2022cmo_cutmix_lt} and mixing components \cite{xu2021bayias,chou2020remix} to boost the performance on tail classes. Zhong \etal \cite{zhong2021mislas_mixup} investigate the effect of mixup for \textit{decoupled learning} \cite{kang2019decoupling}, and find that although mixup can improve calibration when used for \textit{representation learning}, it hurts \textit{classifier learning}. Instead, MiSLaS relies on \textit{class-balanced sampling} and \textit{class-aware label smoothing} to promote good calibration during the second stage. Complementary to this approach, we show that \textit{mixup non-trivially improves both calibration and accuracy by a significant margin when combined with logit adjustment}, in a single stage and without the need for data re-sampling.

\textbf{Ensemble learning.} The ensemble of multiple experts, e.g.~Mixture of Experts \cite{jacobs1991mixture,jordan1994hierarchical_mixture} and Product of Experts (PoE) \cite{hinton2002PoE}, have empirically shown stronger generalization and better calibration \cite{lakshminarayanan2017deep_ensembles} over their single-expert counterpart in the balanced setting \cite{szegedy2015deep_ensemble,kurutach2018ensemble}. These benefits are typically attributed to model diversity, e.g.\ \textit{making diverse mistakes} \cite{dietterich2000ensemble_diverse_errors} or exploring different local minima in the loss landscape \cite{fort2019deep_ensemble_local_minima}. In the long-tailed setting, expert-based methods show promising results to improve the \textit{head-tail trade-off} \cite{cai2021ace}. These approaches typically promote diversity by increasing the KL divergence between expert predictions \cite{wang2020experts,li2022TLC_trustworthy} or by learning on different groups of classes \cite{xiang2020lfme,sharma2020long,cai2021ace}. However, the former methods are not tailored to address the head bias, whereas the latter present limited scalability under low-shot regimes. Closer to our work, 
Li \etal \cite{li2022NCL_nested_collab} learn an ensemble of uniform experts with a combination of \textit{balanced softmax} \cite{ren2020balanced}, nested knowledge distillation \cite{hinton2015knowledge_distillation} and contrastive learning \cite{he2020moco}. Recently, Zhang \etal \cite{zhang2021test} propose to address an unknown label distribution shift under the \textit{transductive learning paradigm}, by learning a set of three skill-diverse experts during training time, and then aggregating them with \textit{self-supervised test-time training}. We instead focus on the traditional setting, where our calibrated ensemble can consistently incorporate prior information about the target distribution if available, or by default, it is provably unbiased for minimizing the balanced error with naive aggregation.


%% file: preliminaries.tex
In this section, we start by describing the problem formulation in Section \ref{sec:problem_formulation}, and give an overview of logit adjustment in Section \ref{sec:logit_adjustment}. Next, we introduce our Balanced Product of Experts in Section \ref{sec:balpoe} and describe how we meet the calibration assumption in Section \ref{sec:calibration}.

\subsection{Problem formulation}
\label{sec:problem_formulation}

In a multi-class classification problem, we are given an unknown distribution $\mathbb{P}$ over some data space $\mathcal{X} \times \mathcal{Y}$ of instances $\mathcal{X}$ and labels $\mathcal{Y} = \{ 1, 2, \dots, C\}$, and we want to find a mapping $f: \mathcal{X} \rightarrow \mathbb{R}^C$ that minimizes the misclassification error 
$ 
\mathbb{P}_{x, y} \left( y \neq \argmax_{j \in \mathcal{Y}} f_{j}(x) \right) 
$, where we denote $ f(x) \equiv [ f_{1}(x),...,f_{C}(x) ] $. In practice, given a sample $\mathcal{D} = \{x_i, y_i\}^{N}_{i=1} \sim \mathbb{P}$, we intend to minimize the empirical risk \cite{vapnik1991empirical_risk_min}, 
$ 
R_{\delta}(f) = \frac{1}{N} \sum_{i=1}^N \ell_{0/1}(y_i, f(x_i))
$,
where $n_j$ is the number of samples for class $j$, $N = \sum^{C}_{j=1} n_j$ and $\ell_{0/1}$ denotes the per-sample misclassification error, also known as $0/1$-loss. As $\ell_{0/1}$ is not differentiable, we typically optimize a surrogate loss, such as the \textit{softmax cross-entropy} (CE),
\begin{align}
    \centering
    \ell(y, f(x)) &= - \log \frac{ e^{f_y(x)} } { \sum_{j \in \mathcal{Y}} {e^{ f_j(x) } } }.
\numberthis \label{eqn:CE_1}    
\end{align}

However, in the long-tailed setting, the training class distribution $\mathbb{P}^{\mathrm{train}}(y)$ is usually highly skewed and cannot be disentangled from the data likelihood $\mathbb{P}^{\mathrm{train}}(x | y)$ by minimizing the misclassification error (or its surrogate loss) \cite{hong2021lade}.  
 Following Hong \etal \cite{hong2021lade}, hereafter we assume a \textit{label distribution shift}, i.e.\ although training and test priors may be different, the likelihood shall remain unchanged, 
\begin{align}
    \centering
    \mathbb{P}^{\mathrm{train}}(y) \neq \mathbb{P}^{\mathrm{test}}(y) \ \ \ \ \ \ \
    \mathbb{P}^{\mathrm{train}}(x | y) = \mathbb{P}^{\mathrm{test}}(x | y). 
     \label{eqn:distribution_shift_likeli}
\end{align}

A special case is when all classes are equally relevant, i.e.\ $\mathbb{P}^{\mathrm{test}}(y) \equiv \mathbb{P}^{\mathrm{bal}}(y) = \frac{1}{C}$, where a more natural metric is the balanced error rate (BER) \cite{chan1998learning,brodersen2010balanced}
\begin{align}
    \textnormal{BER}(f) \equiv 
    \frac{1}{C} \sum_{y \in \mathcal{Y}} \mathbb{P}_{x|y} \left( y \neq \argmax_{j \in \mathcal{Y}} f_{j}(x) \right),
     \label{eqn:ber}
\end{align} 
which is consistent with minimizing the misclassification error under a uniform label distribution.

\subsection{Distribution shift by logit adjustment}
\label{sec:logit_adjustment}

Under a label distribution shift, we can write the training conditional distribution as  
   $
   \mathbb{P}^{\mathrm{train}}(y | x) 
    \propto    \mathbb{P}(x | y) \mathbb{P}^{\mathrm{train}}(y) 
$,
which motivates the use of logit adjustment to remove the long-tailed data bias  \cite{menon2020adjustment,ren2020balanced} 
\begin{align*}
\exp \left[ f_y(x) - \log \mathbb{P}^{\mathrm{train}}(y) \right] 
\propto \frac{\mathbb{P}^{\mathrm{train}}(y | x)}{\mathbb{P}^{\mathrm{train}}(y)} 
\propto \mathbb{P}^{\mathrm{bal}}(y|x),  \numberthis  \label{eqn:distribution_shift_uniform}
\end{align*}
leading to an \textit{adjusted scorer} that is \textit{Fisher-consistent} for minimizing BER \cite{menon2020adjustment}. Importantly, to obtain unbiased predictions via logit adjustment, an ideal (unadjusted) scorer must model the training conditional distribution, i.e. $ e^{f_y(x)} \propto \mathbb{P}^{\mathrm{train}}(y | x) $ \cite{menon2020adjustment}. We observe that, although this condition is unattainable in practice, a weaker but necessary requirement is \textit{perfect calibration} \cite{brocker2009perfect_calibration} for the training distribution. Hereafter we refer to the aforementioned condition as the \textit{\textbf{calibration assumption}}. This observation explicitly highlights the importance of confidence calibration to obtain unbiased logit-adjusted models.

Hong \etal \cite{hong2021lade} generalize the logit adjustment to address an arbitrary label distribution shift by observing that we can swap the training prior for a desired test prior,
\begin{align}
	\exp \left[ f_y(x) + \log \frac{\mathbb{P}^{\mathrm{test}}(y)}{\mathbb{P}^{\mathrm{train}}(y)} \right] 
    \propto \mathbb{P}^{\mathrm{test}}(y | x).
	\numberthis
    \label{eqn:distribution_shift_test_post}
\end{align}

\comment{
Furthermore, Xu \etal \cite{xu2021bayias} extend the logit adjusted loss \cite{menon2020adjustment} to accommodate a known test distribution during training time, by minimizing $
    \ell(y, f(x))  = - \log \frac{ e^{f_y(x) + \log \frac{\mathbb{P}^{\mathrm{train}}(y)}{\mathbb{P}^{\mathrm{test}}(y)} } } { \sum_{j \in \mathcal{Y}} {e^{ f_j(x) + \log \frac{\mathbb{P}^{\mathrm{train}}(j)}{\mathbb{P}^{\mathrm{test}}(j)} } } }. 
$
}

Based on the observation that diverse ensembles lead to stronger generalization \cite{fort2019deep_ensemble_local_minima,wang2020experts}, 
we investigate in the following section \textit{\textbf{how to learn an ensemble, which is not only diverse but also unbiased, in the sense of Fisher consistency, for a desired target distribution.}}

%% file: method.tex
\subsection{Balanced Product of Experts}
\label{sec:balpoe}

We will now introduce our main theoretical contribution. We start from the logit adjustment formulation and introduce a convenient parameterization of the distribution shift, as we intend to create an ensemble of logit-adjusted experts, biased toward different parts of the label distribution.

\textbf{Parameterization of target distributions.} Let us revisit the parametric logit adjustment formulation by Menon \etal \cite{menon2020adjustment}. For $\bm{\tau} \in \mathbb{R}^C$, we observe the following:
\begin{align}
 & \argmax_{y} f_y(x) - \bm{\tau}_y \log \mathbb{P}^{\mathrm{train}}(y) \label{eqn:generalized_bs_1}  \nonumber\\
 & = \argmax_{y} f_y(x) - \log \mathbb{P}^{\mathrm{train}}(y) + \log \mathbb{P}^{\mathrm{train}}(y) ^ {1 - \bm{\tau}_y} \numberthis \\   
 & = \argmax_{y} f_y(x) + \log \frac{\mathbb{P}^{ \bm{\lambda}}(y)}{\mathbb{P}^{\mathrm{train}}(y)} , \numberthis  
 \label{eqn:generalized_bs}
\end{align}
where $\bm{\lambda}_y \equiv 1-\bm{\tau}_y$ and $\mathbb{P}^{ \bm{\lambda}}(y) \equiv \frac{\mathbb{P}^{\mathrm{train}}(y) ^ {\bm{\lambda}_y}}{ \sum_{j \in \mathcal{Y} } \mathbb{P}^{\mathrm{train}}(j) ^ {\bm{\lambda}_j}} $. Note that adjusting the scorer $f_y(x)$ with $\bm{\tau}_y \log \mathbb{P}^\mathrm{train}(y)$ in \eqref{eqn:generalized_bs_1} can be interpreted as performing a distribution shift parameterized by $\bm{\lambda}$. Particularly, the adjusted model shall accommodate for a target distribution $\mathbb{P}^{\bm{\lambda}}(y)$. As we intend to incorporate a controlled bias during training, we define the \textit{generalized logit adjusted loss} (gLA), which is simply the softmax cross-entropy of the adjusted scorer in \eqref{eqn:generalized_bs_1}
\begin{align}
    \ell_{\bm{\tau}}(y, f(x)) 
    &= - \log \frac{ e^{f_y(x) + \bm{\tau}_y \log \mathbb{P}^{\mathrm{train}}(y) } } { \sum_{j \in \mathcal{Y}} {e^{ f_j(x) + \bm{\tau}_j \log \mathbb{P}^{\mathrm{train}}(j)} } } \\ 
    &= \log \left[ 1 + \sum_{j \neq y} e^{f_j(x) - f_y(x) + \Delta^{\bm{\tau}}_{yj}} \right], \numberthis  \label{eqn:gen_LA_loss}
\centering\end{align}
where $ \Delta^{\bm{\tau}}_{yj} = \log \frac{ \mathbb{P}^{\mathrm{train}}(j)^{\bm{\tau}_j} }{ \mathbb{P}^{\mathrm{train}}(y)^{\bm{\tau}_y} } $ defines a pairwise class margin. Intuitively, by increasing the margin, the decision boundary moves away from minority classes and towards majority classes \cite{menon2020adjustment}, as illustrated in Figure~\ref{fig:expected_confidence_and_margin_example}(a). For example, by setting $\bm{\lambda}=\mathbbm{1}$ ($\bm{\tau} = 0$), we obtain the CE loss and a long-tailed distribution as a result. For $\bm{\lambda}=0$ ($\bm{\tau}=\mathbbm{1}$) we model the \textit{uniform distribution}, with the so-called \textit{balanced softmax} (BS) \cite{ren2020balanced,menon2020adjustment}. Interestingly, we obtain an \textit{inverse LT distribution} by further increasing the margin, e.g. for $\bm{\tau} = \mathbbm{2}$ ($\bm{\lambda} = -\mathbbm{1}$). Figure~\ref{fig:expected_confidence_and_margin_example}(b) depicts these scenarios for a long-tailed prior (Pareto distribution). In summary, the proposed loss provides a principled approach for learning specialized expert models and is the building block of the ensemble-based framework to be introduced in the next section.   

\begin{figure}[!t]
	\centering
	\setlength{\tabcolsep}{1pt}
	\begin{tabular}{cc}
		\includegraphics[width=0.44\columnwidth,height=0.42\columnwidth,trim=0 0 0 0,clip
		]{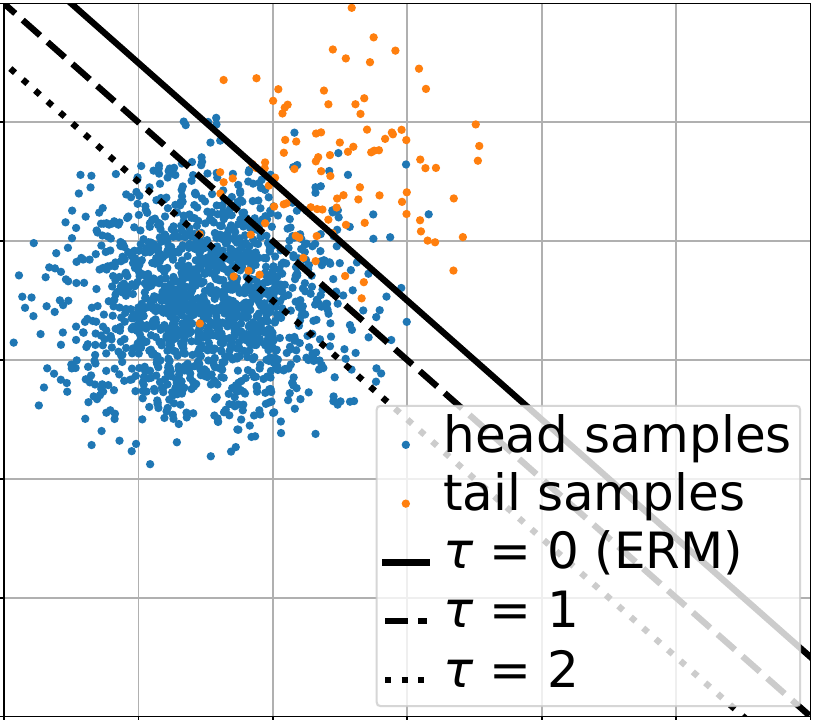} &
		\includegraphics[width=0.5\columnwidth,height=0.42\columnwidth,trim=0 0 8 4,clip] {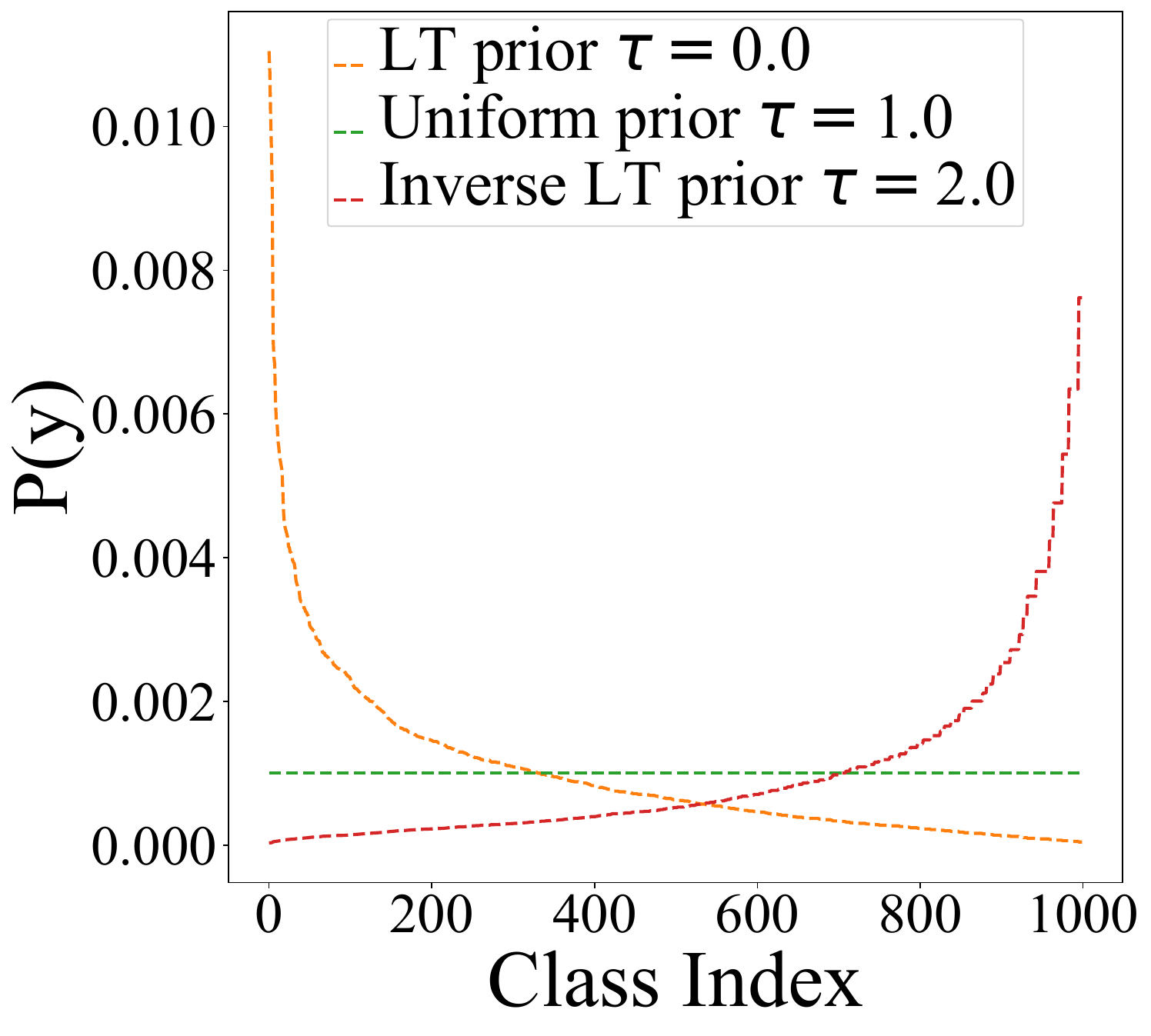} \\
		(a) &  \ \ \ \ \ \ \ \ \ (b)
	\end{tabular}
        \vspace{-6pt}
	\caption{(a) By increasing the pairwise class margin, the decision boundary smoothly moves towards the head class, resulting in (b) diverse label distributions parameterized by $\tau$.}
	\label{fig:expected_confidence_and_margin_example}
  \vspace{-12pt}
\end{figure}

\textbf{Ensemble learning.} To address the long-tailed recognition problem, we introduce  a \textbf{Bal}anced \textbf{P}roduct \textbf{o}f \textbf{E}xperts (BalPoE), which combines multiple \textit{logit-adjusted experts} in order to accommodate for a desired test distribution. Let us denote as $\mathbb{P}^{\bm{\lambda}}(x,y) \equiv \mathbb{P}(x|y) \mathbb{P}^{ \bm{\lambda}}(y)$ the joint distribution associated to $\bm{\lambda} \in \mathbb{R}^C$.

Let $S_{\lambda}$ be a multiset of $\bm{\lambda}$-vectors describing the parameterization of $|S_{\lambda}| \ge 1$ experts. Our framework is composed of a set of scorers $ \{ f^{\bm{\lambda}}\}_{ \bm{\lambda} \in S_{\lambda} } $, where each $\bm{\lambda}$-expert is (ideally) distributed according to $ \exp \left[ f^{\bm{\lambda}}_y(x) \right] \propto \mathbb{P}^{\bm{\lambda}}(x,y)$. The proposed ensemble is then defined as the average of the expert scorers in log space,
\begin{equation}
    \overline{p}{(x,y)} 
    \equiv \exp \left[ \overline{f}_y(x) \right]
    \equiv \exp \left[ \frac{1}{|S_{\lambda}|} \sum_{\bm{\lambda} \in S_{\lambda}} f^{\bm{\lambda}}_y(x) \right].
    \label{eqn:balpoe_def}
\end{equation}

Now, let us denote as $\overline{\bm{\lambda}} \equiv \frac{1}{|S_{\lambda}|} \sum_{\bm{\lambda} \in S_{\lambda}} \bm{\lambda}$ the average of all expert parameterizations. We show, in Theorem \ref{theorem:la_poe_distribution}, that an ensemble consisting of bias-adjusted experts accommodates a marginal distribution 
$\mathbb{P}^{ \overline{\bm{\lambda}}}(y)$.

\begin{theorem}[Distribution of BalPoE]
\label{theorem:la_poe_distribution}
Let $S_{\lambda}$ be a multiset of $\bm{\lambda}$-vectors describing the parameterization of $|S_{\lambda}| \ge 1$ experts. Let us assume dual sets of training and target scorer functions, $ \{ s^{\bm{\lambda}} \}_{ \bm{\lambda} \in S_{\lambda} } $ and $ \{ f^{\bm{\lambda}} \}_{ \bm{\lambda} \in S_{\lambda} } $ with $s, f: \mathcal{X} \rightarrow \mathbb{R}^C$, respectively, s.t. they are related by
\begin{equation}
    \centering
    f^{\bm{\lambda}}_y(x) \equiv s^{\bm{\lambda}}_y(x) - \log \mathbb{P}^{\mathrm{train}}(y) + \bm{\lambda}_y \log \mathbb{P}^{\mathrm{train}}(y).
\end{equation}
Assume that the \textbf{calibration assumption} holds for all training scorers, i.e.\ $\exp s^{\bm{\lambda}}_y(x) \propto \mathbb{P}^{\mathrm{train}}(y | x)$ for $\bm{\lambda} \in S_{\lambda}$. Then, under a label distribution shift in \eqref{eqn:distribution_shift_likeli}, the product of experts as defined in \eqref{eqn:balpoe_def} satisfies
\begin{equation}
\centering
\overline{p}{(x,y)} \propto 
\mathbb{P}(x|y) \mathbb{P}^{ \overline{\bm{\lambda}}}(y) \equiv 
\mathbb{P}^{\overline{\bm{\lambda}}}(x,y).
\end{equation}
\end{theorem}

\begin{proof}
See supplementary material.
\end{proof}

In other words, \textbf{the proposed ensemble attains the average bias of all of its experts}. We can utilize this result to construct an ensemble of diverse experts which is Fisher-consistent for minimizing balanced error. All we need to make sure of is that $\mathbb{P}^{ \overline{\bm{\lambda}}}(y)=\frac{1}{C}$ \cite{menon2020adjustment}. A simple constraint is then $\overline{\bm{\lambda}}_y = 0$ for all $y \in \mathcal{Y}$:

\begin{corollary}[Fisher-consistency]
\label{corollary:fisher_consistency_balpoe}
If $\overline{\bm{\lambda}}_y=0$ for all $y \in \mathcal{Y}$, then the BalPoE scorer $\overline{f}$ is fisher-consistent for minimizing the balanced error.
\end{corollary}
\begin{proof}
From Theorem \ref{theorem:la_poe_distribution}, we have that for $\overline{\bm{\lambda}} = 0$,
\begin{align}
    \argmax_y \overline{f}_y(x) 
    &= \argmax_y \log \mathbb{P}(x|y) \mathbb{P}^{ \overline{\bm{\lambda}}}(y) \\
    &= \argmax_y \log \mathbb{P}(x|y) \frac{1}{C} \\
    &= \argmax_y \mathbb{P}(x|y).
\end{align}
From (7) in \cite{menon2020adjustment}, it follows that BalPoE coincides with the \textit{Bayes-optimal rule} for minimizing the balanced error.
\end{proof}

Thus, by defining a set of $\bm{\lambda}$-experts, such that the average bias parameterization $\overline{\bm{\lambda}} = 0$, we have shown that the resulting ensemble is unbiased. Note that our framework can accommodate for any known target distribution $\mathbb{P}^{\mathrm{test}}(y)$, by ensuring $\overline{\bm{\lambda}}_y = \log \frac{\mathbb{P}^{\mathrm{test}}(y)}{\mathbb{P}^{\mathrm{train}}(y)}$.

\textbf{Training of logit-adjusted experts.} We propose to train each $\bm{\lambda}$-expert by minimizing its respective logit adjusted loss $\ell_{1-\bm{\lambda}}$. The loss of the overall ensemble is computed as the average of individual expert losses, 
\begin{equation}
    \ell^{\mathrm{total}}{(y, \overline{f}(x))} 
    = \frac{1}{|S_{\lambda}|} \sum_{\bm{\lambda} \in S_{\lambda}} \ell_{1-\bm{\lambda}}(y, f^{\bm{\lambda}}(x)).
\end{equation}
During inference, we average the logits of all experts before softmax normalization to estimate the final prediction.

Based on our theoretical observations, we argue that the calibration assumption needs to be met in order to guarantee Fisher consistency for logit-adjusted models. This is however not the case per default in practice, as can be seen in Figure \ref{fig:first_page}(b) for BS (single-expert BalPoE) and in Table \ref{tab:sota_cifar_calibration} for a three-expert BalPoE trained with ERM. Thus, we need to find an effective way to calibrate our ensemble, a matter we discuss in further detail in the next section.

%% file: calibration.tex
\label{sec:calibration}

Following Theorem~\ref{theorem:la_poe_distribution}, we observe that it is desirable for all experts to be well-calibrated for their target distribution in order to obtain an unbiased ensemble. Among calibration methods in the literature, we explore mixup \cite{zhang2017mixup} for two main reasons: (1) mixup has shown to improve calibration in the balanced setting \cite{thulasidasan2019mixup_calibration} and to some extent in the long-tailed setting \cite{zhong2021mislas_mixup}. (2) mixup does not change the prior class distribution $\mathbb{P}^\mathrm{train}(y)$ \cite{xu2021bayias}, which is crucial as it would otherwise perturb the necessary bias adjustments.

Mixup samples from a \textit{vicinity distribution} $\{\tilde{x}_i, \tilde{y}_i\}^{N'}_{i=1} \sim \mathbb{P_\nu}$, where $ \tilde{x}$$=$$\xi x_i + (1 - \xi) x_j$ and $\tilde{y}$$=$$\xi y_i + (1 - \xi) y_j$
are convex combinations of random input-label pairs, $\{(x_i, y_i), (x_j, y_j) \} \sim \mathbb{P}^{\mathrm{train}}_{x, y}$, with $\xi \sim Beta(\alpha, \alpha)$ and $\alpha \in (0, \infty)$. The model is trained by minimizing 
$R_{\nu}(f) = \frac{1}{N'} \sum_{i=1}^{N'} \ell(\tilde{y}_i, f(\tilde{x}_i))$, known as \textit{Vicinal Risk Minimization}.

Zhong \etal \cite{zhong2021mislas_mixup} found that, even though calibration is improved, mixup does not guarantee an increase in test accuracy, and could even degrade classifier performance. However, we found this not to be the case for logit-adjusted models. We argue that this is due to the fact that mixup does not change the prior class distribution $\mathbb{P}^\mathrm{train}(y)$, as proved by Xu \etal \cite{xu2021bayias}. While the long-tailed prior is typically unfavorable, and something one strives to adjust, this property is crucial for methods based on logit adjustment, such as BalPoE. We assume $\mathbb{P}^\mathrm{train}(y)$ is known, and in practice, estimate it from data. If the marginal distribution were to change, for instance by data resampling, the precomputed biases would no longer be accurate, and the ensemble would become biased. In summary, mixup is compatible with BalPoE as a means for improving calibration. 
In Section~\ref{sec:sota_calibration_comparison} we show that mixup significantly improves the calibration of our approach, leading to a \textit{Balanced Product of Calibrated Experts}, see the reliability diagram in Figure~\ref{fig:first_page}(d). We empirically observe that mixup also non-trivially improves the balanced error, which we believe is attributed to the fact that it contributes to fulfilling the calibration assumption, thus guaranteeing a balanced ensemble.

%% file: experiments.tex
\label{sec:all_experiments}

We validate our finding of Theorem \ref{theorem:la_poe_distribution} by measuring test accuracy when varying different aspects of the expert ensemble, including the number of experts, test-time target distributions in terms of $\lambda$, and backbone configuration. Furthermore, we measure model calibration of our expert ensemble and investigate how this is affected by mixup. Finally, we compare our method with current state-of-the-art methods on several benchmark datasets.

\subsection{Experimental setup}
\label{sec:experimental_setup}

\subsubsection{Long-tailed datasets}
\label{sec:datasets}

\textbf{CIFAR-100-LT.} Following previous work \cite{cui2019effective, cao2019learning}, a long-tailed version of the balanced CIFAR-100 dataset \cite{Krizhevsky09learningmultiple} is created, by discarding samples according to an exponential profile. The dataset contains 100 classes, and the class imbalance is controlled by the imbalance ratio (IR) $\rho = \frac{\max_{i} n_i}{\min_{i} n_i} $, i.e.\ the ratio between the number of instances for the most and least populated classes. We conduct experiments with $ \rho \in \{ 10, 50, 100 \} $. For experiments on CIFAR-10-LT \cite{Krizhevsky09learningmultiple}, see the supplementary material.

\textbf{ImageNet-LT.} Built from ImageNet-2012 \cite{deng2009imagenet} with 1K classes, this long-tailed version is obtained by sampling from a Pareto distribution with $\alpha = 6$ \cite{liu2019large}. The resulting dataset consists of 115.8K training, 20K validation, and 50K test images. The categories in the training set contain between 5 and 1280 samples, with an imbalanced ratio of $\rho=256$.

\textbf{iNaturalist.} The iNaturalist 2018 dataset \cite{Horn2017inaturalist} is a large-scale species classification dataset containing 437K images and 8142 classes, with a natural imbalance ratio of $\rho=500$.

\vspace{-4pt}

\subsubsection{Implementation details}
\label{sec:implementation_details}

For the experiments on CIFAR, following Cao \etal \cite{cao2019learning}, we train a ResNet-32 backbone \cite{he2016deep} for 200 epochs with SGD, initial learning rate (LR) of 0.1, momentum rate of 0.9, and a batch size of 128. A multi-step learning rate schedule decreases the LR by a factor of 10 at the 160th and 180th epochs. For large-scale datasets, we follow \cite{wang2020experts,zhang2021test}. For ImageNet-LT, we train ResNet-50 and ResNeXt-50 \cite{xie2017resnext} backbones for 180 epochs with a batch size of 64. For iNaturalist, we train a ResNet-50 for 100 epochs with a batch size of 512. We use a cosine annealing scheduler \cite{loshchilov2017sgdr} for ImageNet-LT and iNaturalist, with initial LR of 0.025 and 0.2, respectively. 
For mixup experiments, we set $\alpha$ to $0.4$, $0.3$, and $0.2$ for CIFAR-100-LT, ImageNet-LT and iNaturalist, respectively, unless stated otherwise. By default, we use 3 experts for BalPoE. We set $S_{\lambda}$ to $\{1,0,-1\}$, $\{1,-0.25,-1.5\}$ and $\{2,0,-2\}$, for CIFAR-100-LT, ImageNet-LT and iNaturalist, respectively, unless stated otherwise. See the supplementary material for additional details.

In addition, we conduct experiments with the longer training schedule using stronger data augmentation, following Cui \etal \cite{Cui2021PaCo}. All models are trained for 400 epochs, using RandAugment \cite{cubuk2020randaugment} in the case of ImageNet and iNaturalist, and AutoAugment \cite{cubuk2018autoaugment} for CIFAR-100-LT. In the latter case, the learning rate is decreased at epochs 320 and 360.

\subsubsection{Evaluation protocol}
\label{sec:evaluation_protocol}

Following Cao \etal \cite{cao2019learning}, we report accuracy on a balanced test set. Whenever confidence intervals are shown, they are acquired from five runs. We also report accuracy for three common data regimes, where classes are grouped by their number of samples, namely, many-shot (> $100$), medium-shot ($20$-$100$), and few-shot (< $20$). 

In addition, we investigate model calibration by estimating the expected calibration error (ECE), maximum calibration error (MCE), and visualizing reliability diagrams. See the supplementary material for definitions of these metrics.

Finally, for a more realistic benchmark, we follow \cite{hong2021lade,zhang2021test} and evaluate our approach under a diverse set of shifted target distributions, which simulate varying levels of class imbalance, from \textit{forward long-tailed} distributions that resemble the training data to extremely different \textit{backward long-tailed} cases. For more details on how these distributions are created, see the supplementary material.

\subsection{State-of-the-art comparison}

In this section, we compare our method, BalPoE, with state-of-the-art methods, in terms of accuracy and calibration, both on the uniform and non-uniform test distributions.

\subsubsection{Comparison under the balanced test distribution}
\label{sec:sota_comparison}

Table \ref{tab:sota_cifar_imagenet_inaturalist} compares accuracy under the balanced test distribution on CIFAR-100-LT, ImageNet-LT, and iNaturalist 2018. For ensemble methods, we report the performance for a three-expert model. For the standard setting, we observe significant improvements over previous methods across most benchmarks. For CIFAR-100-LT, we gain +1.5 and +2.2 accuracy points (pts) over the previous state-of-the-art for imbalance ratios 10 and 100, respectively. On ImageNet-LT, we improve by +3.6 and +1.1 for ResNet50 and ResNeXt50 backbones over RIDE and SADE, respectively, which highlights the benefits of BalPoE compared to other ensemble-based competitors that are not fisher-consistent by design. On iNaturalist, we outperform SADE by +2.1 pts and match the previous SOTA, NCL (+0.1 pts), which trains for longer (4x epochs) with strong data augmentation. Lastly, our skill-diverse ensemble considerably outperforms its single-expert counterpart (LA), while also surpassing an \textit{all-uniform ensemble} ($S_{\lambda}$$=$$\{0, 0, 0\}$) on challenging large-scale datasets.

Similar to PaCo \cite{Cui2021PaCo}, BalPoE also benefits from stronger data augmentation and longer training, commonly required by contrastive learning approaches \cite{Cui2021PaCo,zhu2022BLC_balanced_constrastive,li2022NCL_nested_collab}, setting a new state-of-the-art across all benchmarks under evaluation.

\begin{table}[t]

\centering
\caption{Test accuracy (\%) on CIFAR-100-LT, ImageNet-LT, and iNaturalist 2018 for different imbalance ratios (IR) and backbones (BB). Notation: R32=ResNet32, R50=ResNet50, RX50=ResNeXt50. DA denotes data augmentation. $\star$:  reproduced results. $^{**}$: reproduced with mixup. $\dagger$: From \cite{xu2021bayias}. $\ddagger$: From \cite{zhang2021test}.}
\label{tab:sota_cifar_imagenet_inaturalist}

\resizebox{\columnwidth}{!}{

\begin{tabular}{lcccccc}
    \toprule
    & \multicolumn{3}{c}{CIFAR-100-LT} & \multicolumn{2}{c}{ImageNet-LT} & iNat \\
    \cmidrule(r){2-4}
    \cmidrule(r){5-6}
    \cmidrule(r){7-7}
    \qquad \qquad \qquad \ \ IR $\rightarrow$ & 10 & 50 & 100 & 256 & 256 & 500 \\
    \cmidrule(r){2-4}
    \cmidrule(r){5-6}
    \cmidrule(r){7-7}
    Method $\downarrow$ \qquad BB $\rightarrow$ & R32 & R32 & R32 & R50 & RX50 & R50 \\
    \midrule
    CE$^\star$ & 57.2 & 43.9 & 38.8 & 47.2 & 48.0 & 65.2\\

    CB-Focal \cite{cui2019effective} & 58.0 & 45.3 & 39.6 & - & - & - \\    
    LDAM-DRW \cite{cao2019learning} & 58.7 & 48.0$^\dagger$ & 42.0  & 45.8$^\dagger$ & - & 68.0 \\
    BS \cite{ren2020balanced} & 59.9$^\dagger$ & 49.8$^\dagger$ & 43.9 & 51.1 & - & 68.4 \\
    LA$^{**}$ \cite{menon2013statistical} & - & - & 47.0 & - & 55.2 & 69.9 \\
    
    LADE \cite{hong2021lade} & 61.7 & 50.5 & 45.4 & - & 53.0 & 70.0 \\    
    MiSLAS \cite{zhong2021mislas_mixup} & 63.2 & 52.3 & 47.0 & 52.7 & 51.4$^\ddagger$ & 71.6 \\
    RIDE \cite{wang2020experts} & 61.8$^\ddagger$ & 51.7$^\ddagger$ & 48.0 & \underline{54.9} & 56.4 & 72.2 \\
    DiVE \cite{he2021dive} & 62.0 & 51.1 & 45.4 & 53.1 & - & 71.7 \\
    SSD \cite{li2021SSD} & 62.3 & 50.5 & 46.0 & - & 56.0 & 71.5 \\    
    DRO-LT \cite{samuel2021DRO_LT} & 63.4 & \textbf{57.6} & 47.3 & - & 53.5 & 69.7 \\
    ACE \cite{cai2021ace} & - & 51.9 & 49.6 & 54.7 & 56.6 & \underline{72.9} \\
    UniMix+Bayias \cite{xu2021bayias} & 61.3 & 51.1 & 45.5 & 48.4 & - & 69.2 \\
    CMO+BS \cite{park2022cmo_cutmix_lt} & 62.3 & 51.4 & 46.6 & - & - & 70.9 \\
    TLC \cite{li2022TLC_trustworthy} & - & - & 49.0 & 54.6 & - & - \\
    SADE \cite{zhang2021test} & \underline{63.6} & 53.9 & \underline{49.8} & - & \underline{58.8} & \underline{72.9} \\

    \textbf{Uniform BalPoE} (ours) & - & - & 52.0 & - & 59.5 & 74.7 \\
    \textbf{BalPoE} (ours) & \textbf{65.1} & \underline{56.3} & \textbf{52.0} & \textbf{58.5} & \textbf{59.8} & \textbf{75.0} \\ 
    
    \midrule
    \textit{Stronger DA} \\
    BCL \cite{zhu2022BLC_balanced_constrastive} & 64.9 & 56.6 & 51.9 & 56.0 & 57.1 & 71.8 \\
    \textbf{BalPoE} (ours) & \textbf{66.3} & \textbf{58.7} & \textbf{54.7} & \textbf{59.7} & \textbf{61.6} & \textbf{73.5} \\
    \midrule
    \textit{Longer training} \\
    PaCo \cite{Cui2021PaCo} & 64.2 & 56.0 & 52.0 & 57.0 & 58.2 & 73.2 \\
    CMO+BS \cite{park2022cmo_cutmix_lt} & \underline{65.3} & 56.7 & 51.7 & 58.0 & - & 74.0\\
    BCL \cite{zhu2022BLC_balanced_constrastive} & - & - & 53.9 & - & - & - \\
    NCL \cite{li2022NCL_nested_collab} & - & \underline{58.2} & \underline{54.2} & \underline{59.5} & 60.5 & \underline{74.9}\\
    SADE \cite{zhang2021test} & \underline{65.3} & 57.3 & 52.2 & - & \underline{61.2} & 74.5 \\

    \textbf{BalPoE} (ours) & \textbf{68.1} & \textbf{60.1} & \textbf{55.9} & \textbf{60.8} & \textbf{62.0}  & \textbf{76.9} \\
    
    \bottomrule
\end{tabular}

}

\vspace{-4pt}

\end{table}

\subsubsection{Calibration comparison}
\label{sec:sota_calibration_comparison}

\begin{table}[t]
\renewcommand{\arraystretch}{0.7}
\caption{Expected calibration error (ECE), maximum calibration error (MCE), and test accuracy (ACC) on CIFAR-100-LT-100. $\star$: reproduced results. $\dagger$: from \cite{xu2021bayias}. $\ddagger$: trained with ERM.}

\label{tab:sota_cifar_calibration}
\centering

\resizebox{0.85\columnwidth}{!}{

\begin{tabular}{lccc}
\toprule
& \multicolumn{3}{c}{CIFAR-100-LT-100} \\
\cmidrule(r){2-4}
Method $\downarrow$ & ECE $\downarrow$ & MCE $\downarrow$ & ACC $\uparrow$\\
\midrule
CE$^\star$ & 32.0{\scriptsize $\pm$0.4} & 47.3{\scriptsize $\pm$1.8} & 38.8{\scriptsize $\pm$0.6} \\
Bayias \cite{xu2021bayias} & 24.3 & 39.7 & 43.5 \\
TLC \cite{li2022TLC_trustworthy} & 22.8 & - & 49.0 \\
\textbf{BalPoE}$^\ddagger$ (ours) & 17.6{\scriptsize $\pm$0.4} & 28.9{\scriptsize $\pm$0.9} & 49.2{\scriptsize $\pm$0.5} \\

\midrule

Mixup$^\star$\cite{zhang2017mixup} & 9.6{\scriptsize $\pm$0.8} & 15.9{\scriptsize $\pm$1.5} & 40.8{\scriptsize $\pm$0.7} \\
Remix$^\dagger$\cite{chou2020remix} & 33.6 & 51.0 & 41.9 \\
UniMix+Bayias \cite{xu2021bayias} & 23.0 & 37.4 & 45.5 \\
MiSLAS \cite{zhong2021mislas_mixup} & \textbf{4.8} & - & 47.0 \\

\textbf{BalPoE} (ours) & \textbf{4.9{\scriptsize $\pm$1.0}} & \textbf{11.3{\scriptsize $\pm$1.6}} & \textbf{52.0{\scriptsize $\pm$0.5}} \\
\bottomrule
\end{tabular}

}

\vspace{-12pt}

\end{table}

In this section, we provide a comparison with previous methods for confidence calibration under the LT setting. See Table \ref{tab:sota_cifar_calibration} for ECE, MCE, and accuracy computed over CIFAR-100-LT-100 balanced test set. First, we observe that \textit{Fisher-consistent} approaches trained with ERM, such as Bayias \cite{xu2021bayias} and BalPoE, significantly improve calibration over the standard CE. In this setting, our approach achieves lower calibration errors compared to single-expert and multi-expert competitors, namely, Bayias and TLC \cite{li2022TLC_trustworthy}. Second, we notice that mixup ($\alpha=0.4$) is surprisingly effective for improving calibration under the LT setting. Although Remix \cite{chou2020remix} and UniMix \cite{xu2021bayias} improve tail performance by modifying mixup, they tend to sacrifice model calibration, as shown in Table \ref{tab:sota_cifar_calibration}. Differently from these methods, we show that the regular mixup effectively regularizes BalPoE, simultaneously boosting its generalization and calibration performance. We hypothesize that logit adjustment might benefit from a smoother decision boundary induced by mixup \cite{verma2019manifold}. Finally, although MiSLAS \cite{zhong2021mislas_mixup} leverages mixup for classifier learning, it relies on \textit{decoupled learning} \cite{kang2019decoupling} and \textit{class-aware label smoothing} for improving calibration, while our approach trains end-to-end without incurring in data re-sampling. Remarkably, BalPoE improves generalization without sacrificing confidence calibration, effectively keeping the best of both worlds. 

\subsubsection{Results under diverse test distributions}
\label{sec:sota_comparison_test_distributions}

\begin{table}[t]
\centering

\caption{Test accuracy (\%) on multiple target distributions for CIFAR-100-LT-100 and Imagenet-LT (ResNeXt50). \textit{Prior}: test class distribution is used. $*$: Prior implicitly estimated from test data by self-supervised learning. $\dagger$: results from \cite{zhang2021test}.}
\label{tab:sota_cifar_imagenet_test_distributions}

\resizebox{\columnwidth}{!}{

\begin{tabular}{lccccccccccc}
    \toprule
    & & \multicolumn{5}{c}{CIFAR-100-LT-100} & \multicolumn{5}{c}{Imagenet-LT} \\
    \cmidrule(lr){3-7}
    \cmidrule(lr){8-12}
    & & \multicolumn{2}{c}{Fwd-LT} & Uni & \multicolumn{2}{c}{Bwd-LT} & \multicolumn{2}{c}{Fwd-LT} & Uni & \multicolumn{2}{c}{Bwd-LT} \\     
    \cmidrule(lr){3-4}
    \cmidrule(lr){5-5}
    \cmidrule(lr){6-7}
    \cmidrule(lr){8-9}
    \cmidrule(lr){10-10}
    \cmidrule(lr){11-12}
    
     Method & prior & 50 & 5 & 1 & 5 & 50 & 50 & 5 & 1 & 5 & 50 \\
    \midrule
    Softmax$^\dagger$ & \ding{55} & 63.3 & 52.5 & 41.4 & 30.5 & 17.5 & 66.1  & 56.6 & 48.0 & 38.6 & 27.6 \\    
    
    MiSLAS$^\dagger$ & \ding{55} & 58.8 & 53.0 & 46.8 & 40.1 & 32.1 & 61.6 & 56.3 & 51.4 & 46.1 & 39.5 \\
    
    LADE$^\dagger$ & \ding{55} & 56.0 & 51.0 & 45.6 & 40.0 & 34.0 & 63.4 & 57.4 & 52.3 & 46.8 & 40.7 \\
    
    RIDE$^\dagger$ & \ding{55} & 63.0 & 53.6 & 48.0 & 38.1 & 29.2 & \textbf{67.6} & 61.7 & 56.3 & 51.0 & 44.0 \\
    
    SADE          & \ding{55} & 58.4 & 53.1 & 49.4 & 42.6 & 35.0 & 65.5 & 62.0 & 58.8 & 54.7 & 49.8 \\
    
    \textbf{BalPoE} & \ding{55} & \textbf{65.1} & \textbf{54.8} & \textbf{52.0} & \textbf{44.6} & \textbf{36.1} & \textbf{67.6} & \textbf{63.3} & \textbf{59.8} & \textbf{55.7} & \textbf{50.8} \\

    \midrule
    
    LADE$^\dagger$ & \checkmark & 62.6 & 52.7 & 45.6 & 41.1 & 41.6 & 65.8  & 57.5 & 52.3 & 48.8 & 49.2 \\
    
    SADE           & * & 65.9 & 54.8 & 49.8 & 44.7 & 42.4 & 69.4 & 63.0 & 58.8 & 55.5 & 53.1 \\
    
    \textbf{BalPoE} & \checkmark & \textbf{70.3} & \textbf{59.3} & \textbf{52.0} & \textbf{46.9} & \textbf{46.1} & \textbf{72.5} & \textbf{64.6} & \textbf{59.8} & \textbf{57.2} & \textbf{56.9}\\
    \bottomrule
\end{tabular}

 \vspace{-10pt}

}

\vspace{-10pt}

\end{table}

In this section, we compare model generalization with previous methods under multiple test distributions. We report test accuracy for CIFAR-100-LT-100 and ImageNet-LT in Table~\ref{tab:sota_cifar_imagenet_test_distributions}. Without any knowledge of the target distribution, our approach surpasses strong expert-based baselines, such as RIDE and SADE, in most cases, demonstrating the benefits of our Fisher-consistent ensemble. 

When knowledge about $\mathbb{P}^{\mathrm{test}}$ is available, our framework addresses the distribution shift by ensuring $\overline{\bm{\lambda}}_y = \log \frac{\mathbb{P}^{\mathrm{test}}(y)}{\mathbb{P}^{\mathrm{train}}(y)}$. For an efficient evaluation, we first train an unbiased ensemble and then perform post-hoc logit adjustment to incorporate the test bias \cite{hong2021lade}, instead of training a specialized ensemble for each target distribution. For CIFAR-100-LT, BalPoE surpasses SADE by +2.2 pts under the uniform distribution and by nearly +4.0 pts in the case of highly skewed forward and backward distributions. For the challenging ImageNet-LT benchmark, BalPoE outperforms SADE by similar margins, with a +5.1 and +3.8 pts difference for forward50 and backward50, respectively. See the supplementary material for additional results on CIFAR-100-LT, ImageNet-LT, and iNaturalist datasets.

\subsection{Ablation study and further discussion}

\begin{figure}[!t]
	\centering
	\begin{minipage}[t]{\columnwidth}\vspace{0pt}
	\begin{tabular}{cc}
		\includegraphics[height=0.42\textwidth]{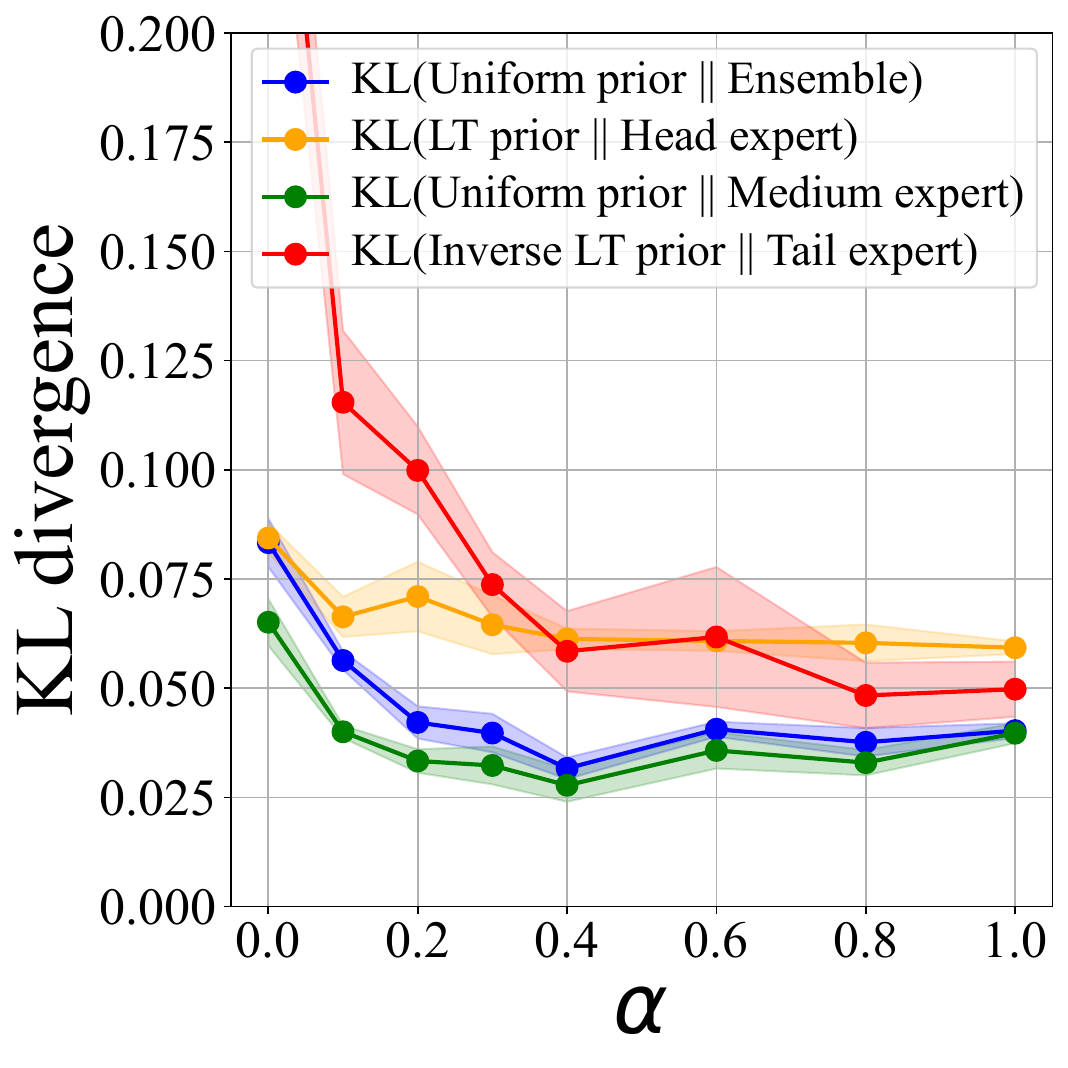} &	
		\includegraphics[height=0.41\textwidth]{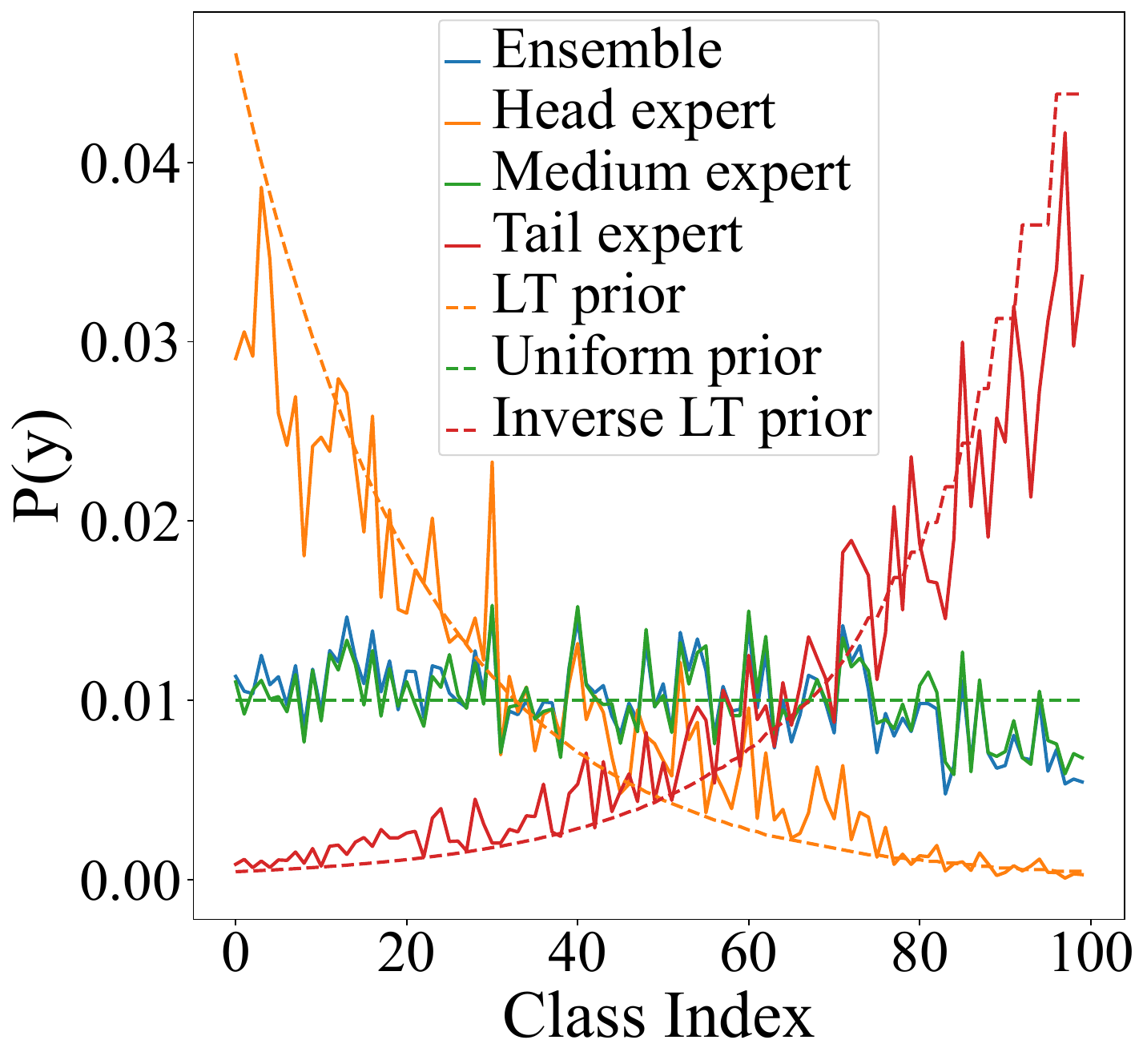} \\
		\ \ \ \ \ \ \ (a) & \ \ \ \ \ (b)
	\end{tabular}

        \vspace{-6pt}
  
	\caption{(a) KL divergence of target priors vs expected marginal. (b) Target priors vs expected marginals. Marginals estimated by averaging predictions over CIFAR-100-LT-100 test set.}
	\label{fig:kl_div_and_expected_confidence}

     \vspace{-4pt}
 
	\end{minipage}

     \vspace{-12pt}
\end{figure}

\textbf{Effect of mixup.}
As seen in Figure \ref{fig:first_page}, mixup reduces the calibration error for logit-adjusted models. We further investigate the effect of mixup by comparing the expected prior (estimated from data) to the ideal prior of each expert. We train BalPoE by setting $S_{\lambda}$ to $\{1, 0, -1\}$ on CIFAR-100-LT-100, and vary the mixup parameter $\alpha$. Figure \ref{fig:kl_div_and_expected_confidence}(a) shows the KL divergence between expert-specific biases, estimated by averaging the predictive confidence of a given $\lambda$-expert over the balanced test data, against the corresponding parametric prior $\mathbb{P}^{\lambda}(y)$. The prior of the ensemble is compared to the uniform distribution. We find that mixup regularization decreases the divergence for all experts, as well as for the ensemble up to $\alpha=0.4$, where the divergence attains its minimum for the uniform distribution. This provides further evidence for the fact that well-regularized experts are better calibrated for their respective target distributions. Figure \ref{fig:kl_div_and_expected_confidence}(b) shows that estimated marginal distributions are reasonable approximations of the ideal priors for $\alpha=0.4$.

\textbf{Do we need a \textit{balanced ensemble}?} To verify the validity of Corollary \ref{corollary:fisher_consistency_balpoe} in practical settings, we vary the average bias $\overline{\lambda}$ and report the test error in Figure \ref{fig:calib_vs_performance_errors}(a). As expected, the optimal choice for $\overline{\lambda}$ is near zero, particularly for well-regularized models. For ERM ($\alpha=0$), the best $\overline{\lambda}$ is -0.25. In this case, the calibration assumption might be violated, thus unbiased predictions cannot be guaranteed.

\textbf{Effect of the number of experts.} We investigate the scalability of our approach in Figure \ref{fig:number_of_experts}, where we plot accuracy over the number of experts. We set different $\lambda$ configurations equidistantly spaced, with minimum and maximum values at $-1.0$ and $1.0$ respectively, and $\lambda=0$ for the single-expert case. This ensures that the average is $\overline{\lambda}=0$. We observe an increase in accuracy as more experts are added to the ensemble. As shown in Figure \ref{fig:first_page}(a), a two-expert BalPoE is a cost-effective trade-off, which surpasses contrastive-learning approaches, e.g. BCL and PaCo, as well as more complex expert approaches, such as NCL and SADE, which rely on knowledge distillation and test-time training, respectively. Interestingly, the performance for even-numbered ensembles interpolates seamlessly, which indicates that a \textit{uniform specialist} is not strictly required. Our approach achieves a peak performance of 57\% for seven experts, which represents a relative increase of 11\% over its single-expert counterpart. Notably, BalPoE provides tangible benefits for many-, medium-, and especially, few-shot classes.

\textbf{Connection to recent approaches.} Our framework generalizes several recent works based on logit adjustment for single-expert models \cite{menon2020adjustment,ren2018learning,hong2021lade,xu2021bayias}. LADE \cite{hong2021lade} introduces a regularizer based on the Donsker-Varadhan representation, which empirically improves single-expert calibration for the uniform distribution, whereas we use mixup regularization to meet the calibration assumption. Within our theoretical framework, we observe that SADE \cite{zhang2021test}, which learns forward, uniform, and backward experts (with different softmax losses), is neither well-calibrated (see Figure \ref{fig:first_page}(c)) nor guaranteed to be Fisher-consistent, as in general, a backward bias (based on flipped class probabilities) does not cancel out with a forward bias, particularly without test-time aggregation. Finally, NCL \cite{li2022NCL_nested_collab} learns an ensemble of uniform experts with a combination of balanced softmax, contrastive learning, and nested distillation, but does not enforce the calibration assumption required to achieve Fisher consistency. 

\begin{figure}[!t]
	\centering
	\begin{minipage}[t]{\columnwidth}\vspace{0pt}
        \centering
        \begin{tabular}{cc}
		    \includegraphics[height=0.42\textwidth]{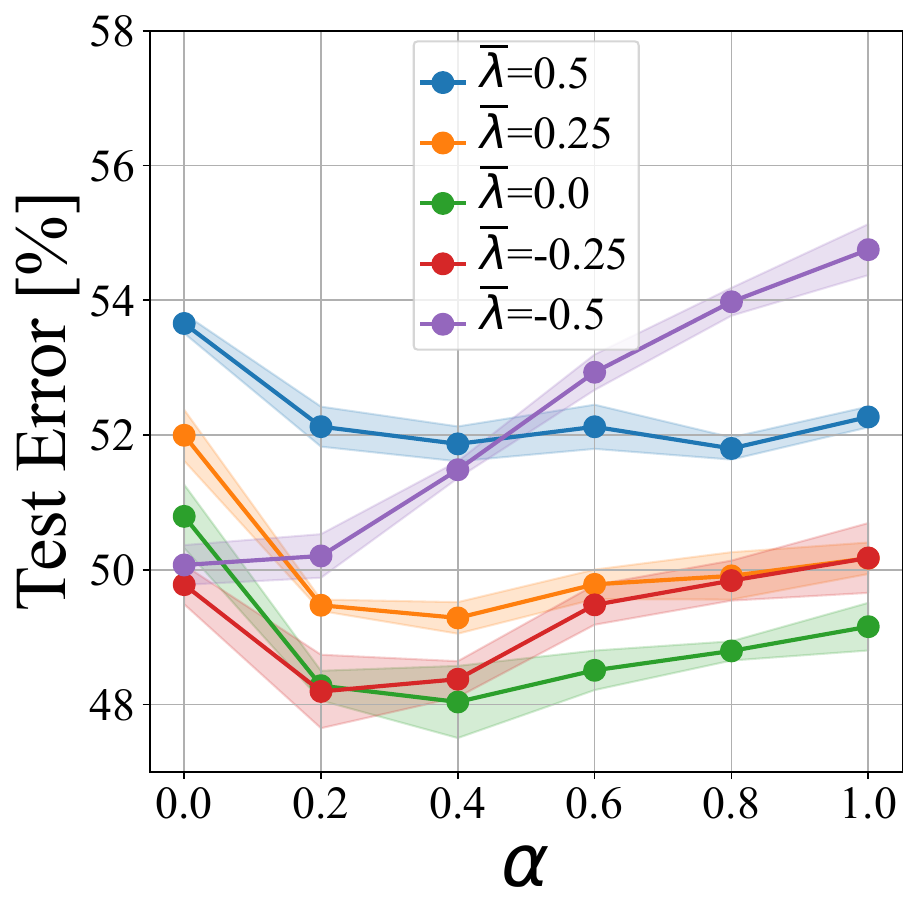} &
		    \includegraphics[height=0.42\textwidth]{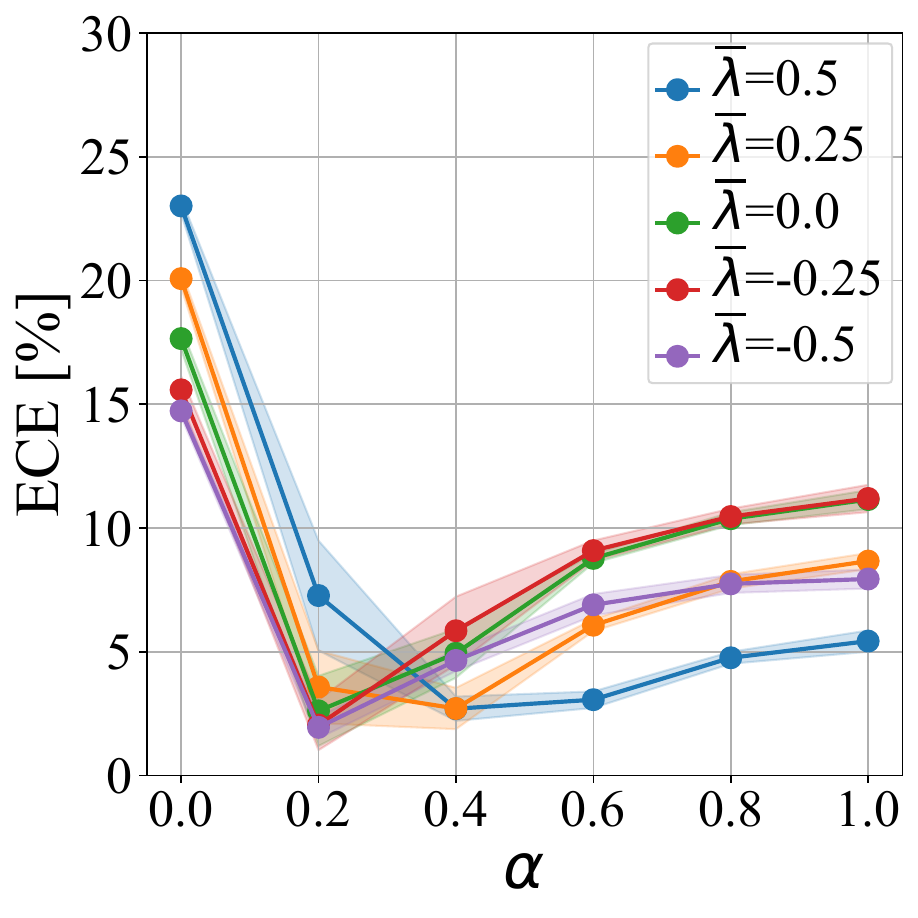} \\
		    (a) & (b)
	    \end{tabular}
            \vspace{-6pt}
	    \caption{(a) Test error and (b) expected calibration error (ECE) computed over CIFAR-100-LT-100 balanced test set.}       
	    \label{fig:calib_vs_performance_errors}
    
    \end{minipage}%
    \vspace{-12pt}
\end{figure}

\begin{figure}
    \setlength{\tabcolsep}{0pt}
	\centering
	\begin{tabular}{cccc}
		\includegraphics[height=0.2\textwidth]{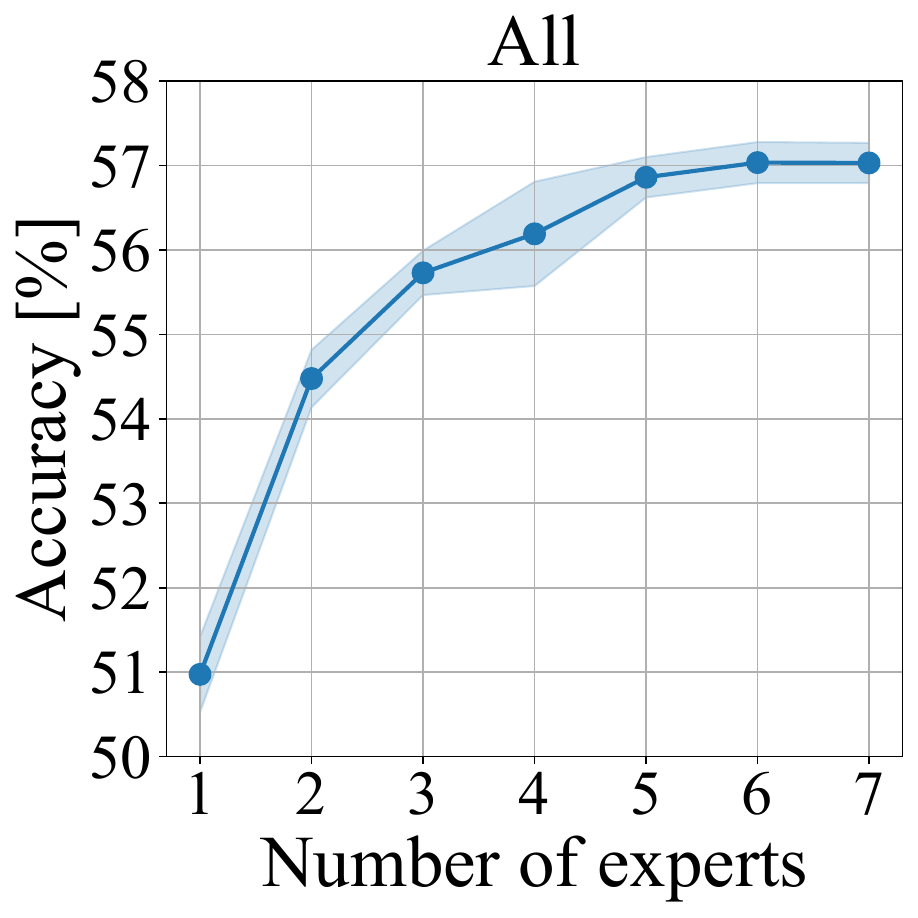} &
		\includegraphics[height=0.2\textwidth
		]{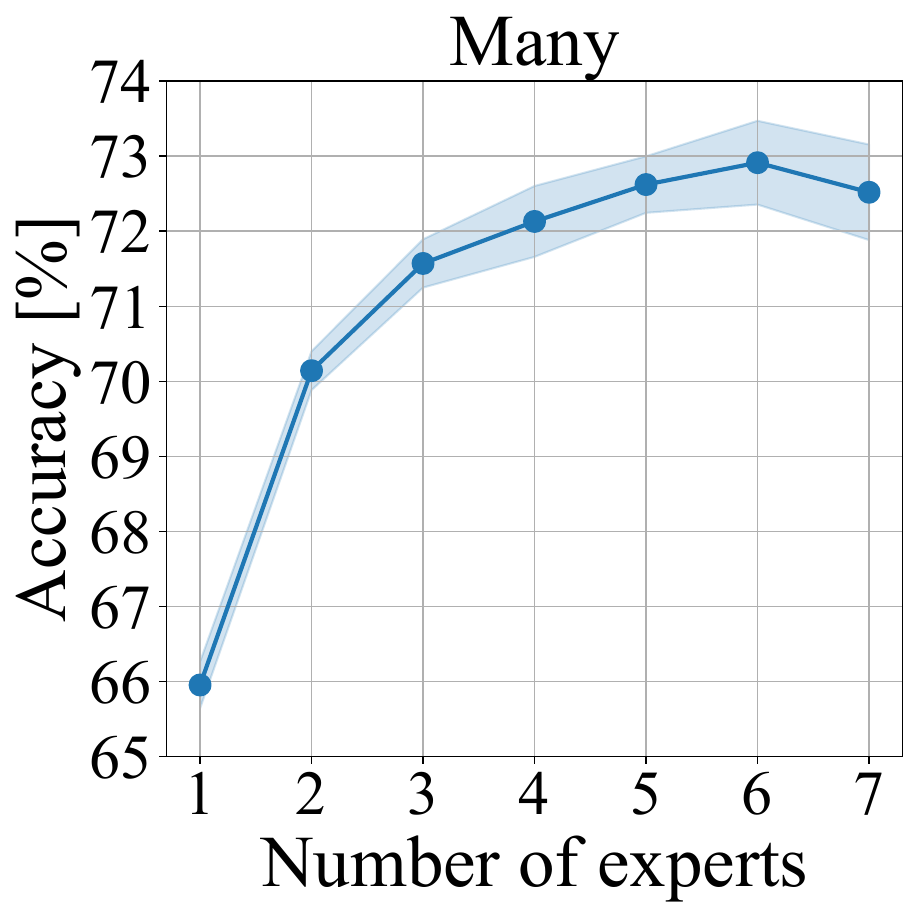} \\
		\includegraphics[height=0.2\textwidth
		]{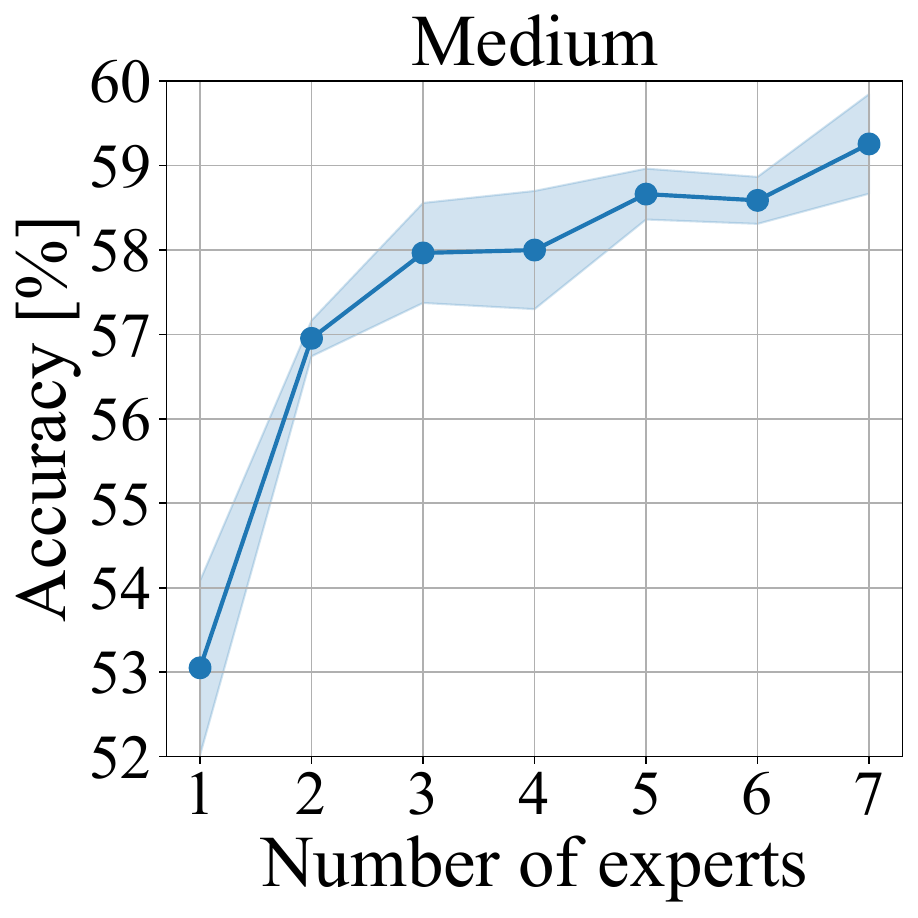} &
		\includegraphics[height=0.2\textwidth
		]{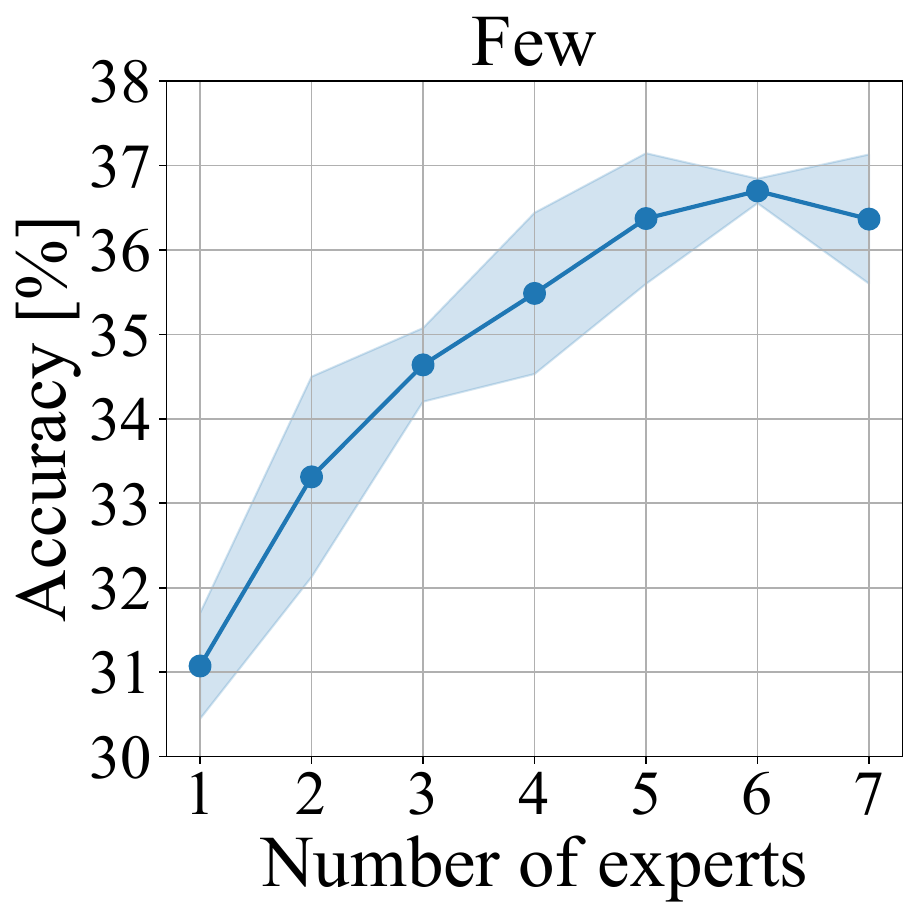} \\		
	\end{tabular}

        \vspace{-6pt}
        
	\caption{Test accuracy vs number of experts on CIFAR-100-LT-100 for all, many-shot, medium-shot and few-shot classes.}

 \vspace{-6pt}

\label{fig:number_of_experts}

 \vspace{-10pt}
 
\end{figure}

\vspace{-2pt}

%% file: conclusions.tex
\label{sec:discussion}
In this paper, we extend the theoretical foundation for logit adjustment to be used for training a balanced product of experts (BalPoE). We show that the ensemble is Fisher-consistent for the balanced error, given that a constraint for the expert-specific biases is fulfilled. We find that model calibration is vital for achieving an unbiased ensemble since the experts need to be weighed against each other in a proper way. This is achieved using mixup. Our BalPoE sets a new state-of-the-art on several long-tailed benchmark datasets.

\textbf{Limitations.} First, we assume $\mathbb{P}^{\mathrm{train}}(x|y)=\mathbb{P}^{\mathrm{test}}(x|y)$, which is a fair assumption but may be violated in practice, e.g. in autonomous driving applications, where the model might be exposed to out-of-distribution data.
Second, the prior $\mathbb{P}^{\mathrm{train}}(y)$ is estimated empirically based on the number of training samples, which can be suboptimal for few-shot classes. To address this issue, considering the effective number of samples \cite{cui2019effective} could be an interesting venue for future research. Finally, our findings are limited to single-label multi-class classification, and extending our balanced product of experts to multi-label classification and other detection tasks is left to future work.

\textbf{Societal impact.} We believe that the societal impacts of this work are mainly positive. Our proposed method reduces biases caused by label imbalance in the training data, which is important from an algorithmic fairness point of view. Additionally, our model is calibrated, which improves trustworthiness, and usefulness in downstream tasks. However, it cannot account for out-of-distribution data or cases where a view of a tail class appears at test time, which is not captured in the training data. Thus, there is a risk that the model is being overtrusted. Finally, we utilize multiple models, which increases computational cost and electricity consumption, especially for large-scale datasets.

%% file: acknowledgments.tex
This work was supported by the Wallenberg Artificial Intelligence, Autonomous Systems and Software Program (WASP), funded by Knut and Alice Wallenberg Foundation. The computational resources were provided by the National Academic Infrastructure for Supercomputing in Sweden (NAISS), partially funded by the Swedish Research Council through grant agreement no. 2022-06725, and by the Berzelius resource, provided by the Knut and Alice Wallenberg Foundation at the National Supercomputer Centre.

%% file: appendix.tex
\appendix

\vspace{-20pt}
\begin{center}
\textbf{\Large Supplementary Material for}
\end{center}
\begin{center}\textbf{\Large Balanced Product of Calibrated Experts for Long-Tailed Recognition}
\end{center}

\section{Theoretical results}

\paragraph{Theorem 1 (Distribution of BalPoE)}  Let $S_{\lambda}$ be a multiset of $\bm{\lambda}$-vectors describing the parameterization of $|S_{\lambda}| \ge 1$ experts. Let us assume dual sets of training and target scorer functions, $ \{ s^{\bm{\lambda}} \}_{ \bm{\lambda} \in S_{\lambda} } $ and $ \{ f^{\bm{\lambda}} \}_{ \bm{\lambda} \in S_{\lambda} } $ with $s, f: \mathcal{X} \rightarrow \mathbb{R}^C$, respectively, s.t. they are related by
\begin{equation}
    \centering
    f^{\bm{\lambda}}_y(x) \equiv s^{\bm{\lambda}}_y(x) - \log \mathbb{P}^{\mathrm{train}}(y) + \bm{\lambda}_y \log \mathbb{P}^{\mathrm{train}}(y).
    \label{eqn:logit_adjustment_assumption}
\end{equation}
Assume that the \textbf{calibration assumption} holds for all training scorers, i.e. 
\begin{equation}
\centering
\exp s^{\bm{\lambda}}_y(x) \propto \mathbb{P}^{\mathrm{train}}(y | x)  \ \ \ \ \forall \bm{\lambda} \in S_{\lambda}.
\label{eqn:perfect_training_assumption_propto}
\end{equation}
Then, under a label distribution shift, our product of experts satisfies
\begin{equation}
\centering
\overline{p}{(x,y)} \propto 
\mathbb{P}(x|y) \mathbb{P}^{ \overline{\bm{\lambda}}}(y) \equiv 
\mathbb{P}^{\overline{\bm{\lambda}}}(x,y).
\end{equation}

\paragraph{Proof of Theorem 1}
Given $x \in \mathcal{X}$ and $y \in \mathcal{Y}$, for each $\bm{\lambda} \in S_{\lambda}$ and its respective (training) scorer $s^{\bm{\lambda}}$, we have that
\begin{equation}
\mathbb{P}^{\mathrm{train}}(y | x) = \frac{\exp s^{\bm{\lambda}}_y(x) }{Z^{\bm{\lambda}}_x},  \ \ \ \ \ \ \ \ \ \ \ \ \ \ \ \ \ \ \ \ \ \ \ \  \eqref{eqn:perfect_training_assumption_propto}
\label{eqn:perfect_training_assumption}
\end{equation}
where $Z^{\bm{\lambda}}_x \in \mathbb{R}$ is an (unknown) normalizing factor. Then, our mean scorer $\overline{f}$ satisfies
\begin{align}
    \overline{f}_y(x)
    &= \frac{1}{|S_{\lambda}|} \sum_{\bm{\lambda} \in S_{\lambda}} \left[ s^{\bm{\lambda}}_y(x) + (\bm{\lambda}_y - 1) \log \mathbb{P}^{\mathrm{train}}(y) \right]   \ \ \ \ \ \ \ \ \ \ \ \ \ \ \ \ \ \ \ \ \ 
    \eqref{eqn:balpoe_def},\eqref{eqn:logit_adjustment_assumption} \\ 
    &= \frac{1}{|S_{\lambda}|} \sum_{\bm{\lambda} \in S_{\lambda}} s^{\bm{\lambda}}_y(x) + \left[ \frac{1}{|S_{\lambda}|} \sum_{\bm{\lambda} \in S_{\lambda}} \bm{\lambda}_y - 1 \right] \log \mathbb{P}^{\mathrm{train}}(y) \\ 
    &= \frac{1}{|S_{\lambda}|} \sum_{\bm{\lambda} \in S_{\lambda}} \log \left[  \mathbb{P}^{\mathrm{train}}(y|x) Z^{\bm{\lambda}}_x \right] + (\overline{\bm{\lambda}}_y - 1) \log \mathbb{P}^{\mathrm{train}}(y) \ \ \ \ \eqref{eqn:perfect_training_assumption} \\ 
    &= \log \frac{\mathbb{P}^{\mathrm{train}}(y|x)}{\mathbb{P}^{\mathrm{train}}(y)} + \log \mathbb{P}^{\mathrm{train}}(y)^{\overline{\bm{\lambda}}_y} + \frac{1}{|S_{\lambda}|} \sum_{\bm{\lambda} \in S_{\lambda}} \log Z^{\bm{\lambda}}_x \\
    &= \log \mathbb{P}^{\mathrm{train}}(x|y) + \log \mathbb{P}^{\overline{\bm{\lambda}}}(y) + \overline{C}^{\bm{\lambda}}_x \ \ \ \ \ \ \ \ \ \ \ 
    \textnormal{(see definition of } \overline{C}^{\bm{\lambda}}_x \textnormal{ below)} \\
    &= \log \left[ \mathbb{P}(x|y) \mathbb{P}^{\overline{\bm{\lambda}}}(y) \right]  + \overline{C}^{\bm{\lambda}}_x \ \ \ \ \ \ \ \  
    \\
    &= \log \mathbb{P}^{\overline{\bm{\lambda}}}(x, y) + \overline{C}^{\bm{\lambda}}_x, \ \ \ \ \ \ \ \ \ \ \ \ \ \ \ \ \ 
    \eqref{eqn:distribution_shift_likeli} \\ 
    \label{eqn:log_joint_distribution}
\end{align}
where 
$\overline{C}^{\bm{\lambda}}_x = - \log \mathbb{P}^{\mathrm{train}}(x) + \log \left[ \sum_{j \in \mathcal{Y}} \mathbb{P}^{\mathrm{train}}(j)^{\overline{\bm{\lambda}}_j } \right] + \frac{1}{|S_{\lambda}|} \sum_{\bm{\lambda} \in S_{\lambda}} \log Z^{\bm{\lambda}}_x $ 
hides terms that are constant w.r.t. $y$. By re-arranging terms in \eqref{eqn:log_joint_distribution} and applying \textit{softmax}, $\overline{C}^{\bm{\lambda}}_x$ is cancelled out, obtaining  
\begin{align}
    \frac{\mathbb{P}^{\overline{\bm{\lambda}}}(x, y)}{ \sum_{j \in \mathcal{Y}} \mathbb{P}^{\overline{\bm{\lambda}}}(x, j) } 
    &= \frac{\exp \left[ \overline{f}_y(x) - \overline{C}^{\bm{\lambda}}_x \right]} { \sum_{j \in \mathcal{Y}} \exp \left[ \overline{f}_j(x) - \overline{C}^{\bm{\lambda}}_x \right]} \\
    &= \frac{\exp \overline{f}_y(x)}{ \sum_{j \in \mathcal{Y}} \exp \overline{f}_j(x) } \\
    &= \frac{\overline{p}{(x,y)}}{ \sum_{j \in \mathcal{Y} } \overline{p}{(x,j)} }.
    \label{eqn:normalized_poe_distribution}
\end{align}
From \eqref{eqn:normalized_poe_distribution} it follows that our BalPoE is proportional to a joint (target) distribution parameterized by $\overline{\bm{\lambda}}$, i.e. 
$\overline{p}{(x,y)} \propto \mathbb{P}^{\overline{\bm{\lambda}}}(x,y)$.

\section{Implementation details}

\subsection{Dataset summary}

In Table \ref{tab:datasets} we include additional information for the datasets used in this work.

\begin{table}[ht]
\renewcommand{\arraystretch}{0.5}
\centering
\caption{Summary of long-tailed datasets.}
\label{tab:datasets}
\centering
\begin{tabular}{lccc}
\toprule
Dataset & \# classes & \# samples & IR \\
\midrule
CIFAR-LT \cite{cao2019learning} & 10 / 100 & 60K & \{10,50,100\} \\
ImageNet-LT \cite{liu2019large} & 1K & 186K & 256 \\
iNaturalist 2018 \cite{Horn2017inaturalist} & 8K & 437K & 500 \\
\bottomrule
\end{tabular}
\end{table}

\subsection{Training details}

Following previous LT approaches \cite{cao2019learning,liu2019large}, we use cosine classifier, which is defined as $\psi(z, y) = \frac{\kappa w_y^{T}  z}{||w_y|| ||z||}$, where $w_y$ are learnable weights for class $y$, $z$ denotes the output of a neural network and $\kappa$ is a hyperparameter (set to $32$). We use weight decay with its hyperparameter set to $5 \cdot 10^{-4}$, $5 \cdot 10^{-4}$ and $2 \cdot 10^{-4}$ for CIFAR-LT, ImageNet-LT, and iNaturalist datasets, respectively. For the CIFAR-LT experiments, we use a warmup period of 5 and 10 epochs for standard and longer training schedules, respectively. We use (up to) four Nvidia A100 40GB GPUs to train our models in an internal cluster. 
Following \cite{wang2020experts,cai2021ace}, our expert architecture comprises an extensive shared backbone and small expert heads (one and two ResNet blocks for large-scale and CIFAR experiments, respectively).

\section{Experiments}

Here we present additional experiments and an extended analysis to further validate our approach.

\subsection{Extended state-of-the-art comparison}

In this section, we provide a more detailed comparison with previous state-of-the-art approaches, by reporting test accuracy for many-, medium- and few-shot classes, separately. Tables \ref{tab:sota_cifar_regimes}, \ref{tab:sota_imagenet_regimes} and \ref{tab:sota_inaturalist_regimes} present results for CIFAR-100-LT-100, ImageNet-LT and Inaturalist, respectively. For CIFAR-100-LT-100, see Table \ref{tab:sota_cifar_regimes}, we observe that our balanced product of calibrated experts significantly improves generalization under few-shot and medium-shot regimes, with only a slight drop in head performance, effectively mitigating the elusive head-tail trade-off. Under the standard setting, we surpass all baselines in medium-, few-shot, and overall performance, while also retaining competitive performance in many-shot classes. 
As discussed in the main paper, BalPoE can effectively tackle large-scale datasets. We achieve a new state-of-the-art for few-shot, medium-shot, and overall performance for Inaturalist, see Table \ref{tab:sota_inaturalist_regimes}. Finally, for ImageNet-LT we obtain very strong results, on medium- and few-shot classes on-par with current SOTA approaches, while achieving the best head-tail trade-off in overall performance, as shown in Table \ref{tab:sota_imagenet_regimes}.

\begin{table*}[ht]
\caption{Test accuracy (\%) of ResNet-32 trained on CIFAR-100-LT-100 for methods under comparison. $\star$: Our reproduced results. $\ddagger$: ACE trained for 400 epochs with regular data augmentation.}
\label{tab:sota_cifar_regimes}
\centering
\begin{tabular}{lcccc}
\toprule
& \multicolumn{4}{c}{CIFAR-100-LT-100} \\
\cmidrule(r){2-5}
Methods & Many & Medium & Few & All \\
\midrule
CE$^\star$ & 67.6{\scriptsize $\pm$1.0} & 36.7{\scriptsize $\pm$1.2} & 7.6{\scriptsize $\pm$0.6} & 38.8{\scriptsize $\pm$0.6}  \\
LDAM-DRW \cite{cao2019learning} & - & - & - & 39.6 \\
BS \cite{ren2020balanced} & 59.5 & 45.4 & 30.7 & 46.1 \\
LADE \cite{hong2021lade} & 58.7 & 45.8 & 29.8 & 45.6 \\
MiSLAS \cite{zhong2021mislas_mixup} & 60.4 & 49.6 & 26.6 & 47.0 \\
RIDE \cite{wang2020experts} & 68.1 & 49.2 & 23.9 & 48.0 \\
UniMix+Bayias \cite{xu2021bayias} & - & - & - & 48.4 \\
DRO-LT \cite{samuel2021DRO_LT} & 64.7 & 50.0 & 23.8 & 47.3 \\
TLC \cite{li2022TLC_trustworthy} & \textbf{70.9} & 47.9 & 28.1 & 49.0 \\
SADE \cite{zhang2021test} & 65.4 & 49.3 & 29.3 & 49.8 \\
\textbf{BalPoE} (ours) & 67.7{\scriptsize $\pm$0.3} & \textbf{54.2{\scriptsize $\pm$0.9}} & \textbf{31.0{\scriptsize $\pm$0.6}} & \textbf{52.0{\scriptsize $\pm$0.5}}  \\
\midrule
\textit{Longer training} \\
\midrule
ACE$^\ddagger$ \cite{cai2021ace} & 66.1 & 55.7 & 23.5 & 49.4 \\
PaCo \cite{Cui2021PaCo} & - & - & -& 52.0 \\
BCL \cite{zhu2022BLC_balanced_constrastive} & 69.7 & 53.8 & 35.5 & 53.9 \\
NCL \cite{li2022NCL_nested_collab} & - & - & - & 54.2 \\
SADE \cite{zhang2021test} &  - &  -& - & 52.2 \\

\textbf{BalPoE} (ours) & \textbf{71.4{\scriptsize $\pm$0.6}} & \textbf{58.0{\scriptsize $\pm$0.7}} & \textbf{35.4{\scriptsize $\pm$0.4}} & \textbf{55.9{\scriptsize $\pm$0.4}} \\

\bottomrule
\end{tabular}
\end{table*}

\begin{table*}[ht]
\caption{Test accuracy (\%) of ResNet-50 / ResNeXt-50 trained on ImageNet-LT for methods under comparison. $\star$: Our reproduced results.}
\label{tab:sota_imagenet_regimes}
\centering
\begin{tabular}{lcccccccc}
\toprule
& \multicolumn{8}{c}{ImageNet-LT} \\
\cmidrule(r){2-9}
& \multicolumn{4}{c}{ResNet50} & \multicolumn{4}{c}{ResNeXt50} \\
\cmidrule(r){2-5}
\cmidrule(r){6-9}
Methods & Many & Medium & Few & All & Many & Medium & Few & All \\
\midrule
CE$^\star$ & 66.5 & 40.5 & 15.9 & 47.2 & 68.1 & 41.5 & 14.0 & 48.0 \\
BS \cite{ren2020balanced} & - & - & - & - & 64.1  & 48.2 & 33.4 & 52.3 \\
LADE \cite{hong2021lade} & - & - & - & - & 65.1 & 48.9 & 33.4 & 53.0 \\
MiSLAS \cite{zhong2021mislas_mixup}  & 61.7  & 51.3  & 35.8  & 52.7 & - & - & - & - \\
RIDE \cite{wang2020experts}  & 66.2 & 51.7 & 34.9 & 54.9 & 67.6 & 53.5 & 35.9 & 56.4 \\
ACE \cite{cai2021ace} & - & - & - & 54.7 & - & - & - & 56.6 \\
TLC \cite{li2022TLC_trustworthy} & \textbf{69.3} & \textbf{56.7} & 37.9 & 54.6 & - & - & - & - \\
SADE \cite{zhang2021test} & - & - & - & - & 66.5 & 57.0 & 43.5 & 58.8 \\

\textbf{BalPoE} (ours) & 66.0 & \textbf{56.7} & \textbf{43.6} & \textbf{58.5} & \textbf{68.2} & \textbf{57.2} & \textbf{44.9} & \textbf{59.8} \\

\midrule

\textit{Longer training} \\
\midrule
PaCo \cite{Cui2021PaCo} & 65.0 & 55.7 & 38.2 & 57.0 & 67.5 & 56.9 & 36.7 & 58.2 \\
NCL \cite{li2022NCL_nested_collab} & - & - & - & 59.5 & - & - & - & 60.5 \\
SADE \cite{zhang2021test} & - & - & - & - & 67.3 & \textbf{60.4} & \textbf{46.4} & 61.2 \\
\textbf{BalPoE} (ours) & 67.8 & \textbf{59.2} & \textbf{46.5} & \textbf{60.8} & \textbf{70.8} & 59.5 & \textbf{46.4} &\textbf{62.0} \\
\bottomrule
\end{tabular}
\end{table*}

\begin{table}[ht]
\caption{Test accuracy (\%) of ResNet-50 trained on Inaturalist-2018 for methods under comparison.  $\star$: Our reproduced results.}
\label{tab:sota_inaturalist_regimes}
\centering
\begin{tabular}{lcccc}
\toprule
& \multicolumn{4}{c}{Inaturalist} \\
\cmidrule(r){2-5}
Methods & Many & Medium & Few & All \\
\midrule
CE$^\star$ & \textbf{76.4} & 66.8 & 60.1 & 65.2 \\
LDAM-DRW \cite{cao2019learning} & - & - & - & 68.0 \\
BS \cite{ren2020balanced} & 70.9 & 70.7 & 70.4 & 70.6 \\
LADE \cite{hong2021lade} & 68.9 & 68.7 & 70.2 & 69.3 \\
MiSLAS \cite{zhong2021mislas_mixup} & 73.2 & 72.4 & 70.4 & 71.6 \\
RIDE \cite{wang2020experts} & 70.2 & 72.2 & 72.7 & 72.2 \\
ACE \cite{cai2021ace} & - & - & - & 72.9 \\
SADE \cite{zhang2021test} & 74.5 & 72.5 & 73.0 & 72.9 \\

\textbf{BalPoE} (ours) & 73.2 & \textbf{75.5}  & \textbf{74.7} &  \textbf{75.0} \\
\midrule
\textit{Longer training} \\
\midrule
PaCo \cite{Cui2021PaCo} & 70.3 & 73.2 & 73.6 & 73.2 \\
NCL \cite{li2022NCL_nested_collab} & 72.7 & 75.6 & 74.5 & 74.9 \\
SADE \cite{zhang2021test} & \textbf{75.5} & 73.7 & 75.1 & 74.5 \\
\textbf{BalPoE} (ours) & \underline{75.0} & \textbf{77.4} & \textbf{76.9} & \textbf{76.9} \\
\bottomrule
\end{tabular}
\end{table}

\paragraph{Mixup encourages expert specialization}

We plot the test accuracy for CIFAR-100-LT100 as a function of $\alpha$ in Figure~\ref{fig:alpha_vs_test_error_ensemble}. Results are shown for the final ensemble as well as for the different experts separately, and on different data regimes. We observe that mixup promotes expert specialization, especially for the tail expert which becomes a specialist in few-shot classes. Expert regularization boosts the performance of the ensemble, attaining its peak performance at $\alpha=\text{0.2-0.4}$. This observation is consistent with the study of mixup under the balanced setting \cite{zhang2017mixup}, and previous findings suggesting that large $\alpha$ values may lead to underfitting, due to \textit{manifold intrusion} \cite{guo2019out_of_manifold_intrusion}.

\begin{figure*}[ht]
	\centering
	\setlength{\tabcolsep}{0pt}
	\begin{tabular}{cccc}
		\includegraphics[height=4.3cm]{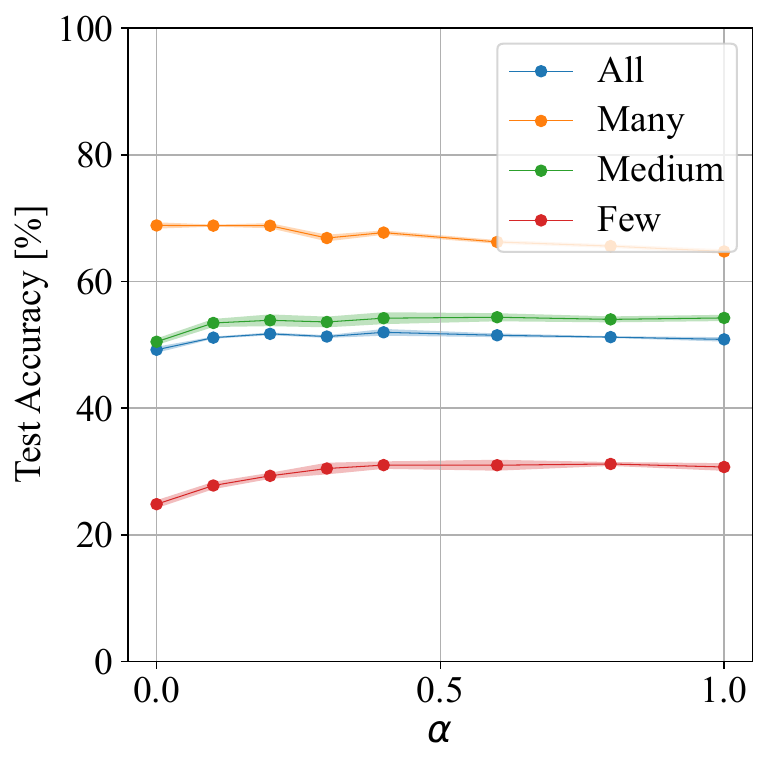} &
		\includegraphics[height=4.3cm,trim=30 0 0 0, clip]{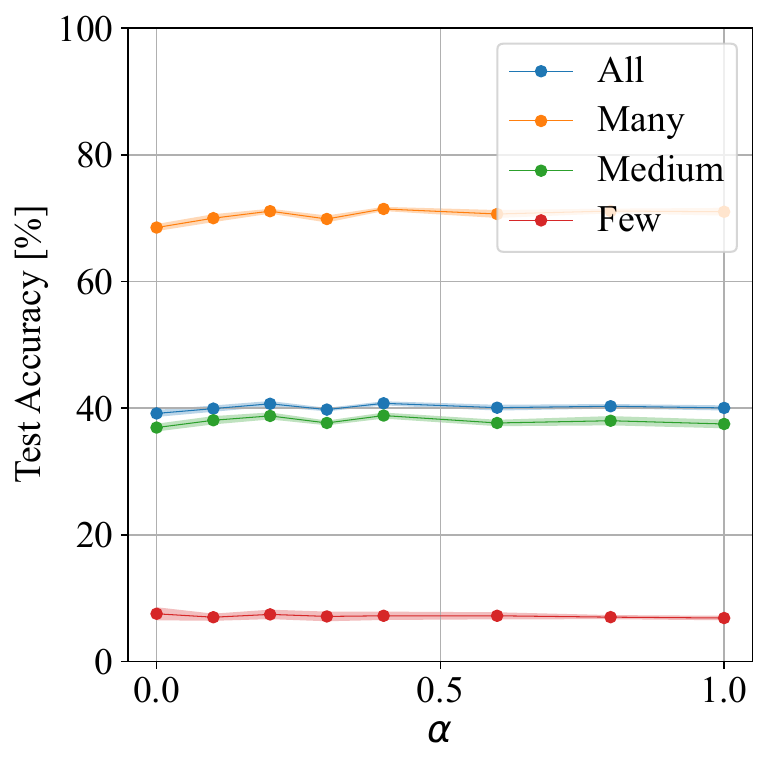} &
		\includegraphics[height=4.3cm,trim=30 0 0 0, clip]{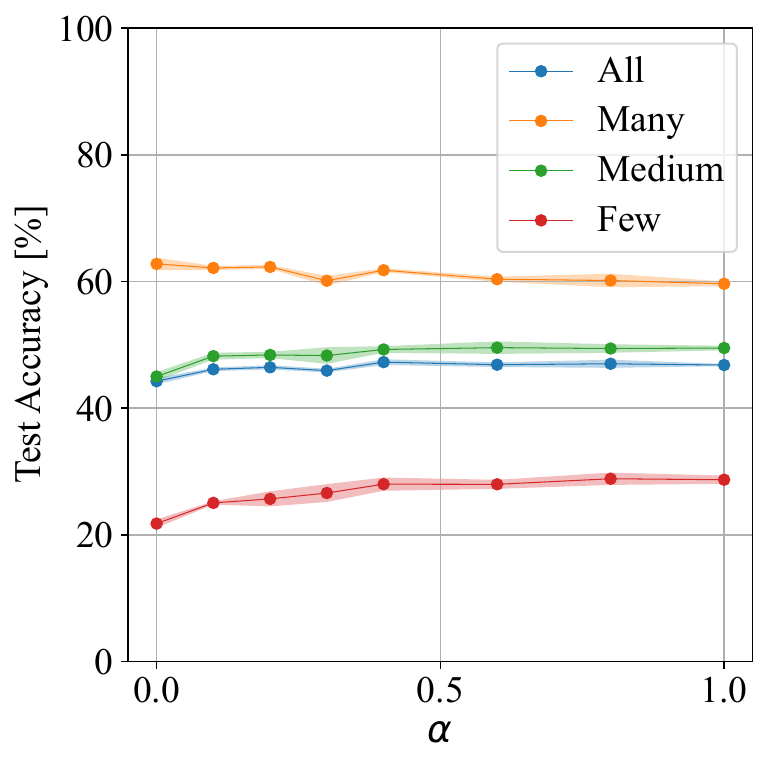} &
		\includegraphics[height=4.3cm,trim=30 0 0 0, clip]{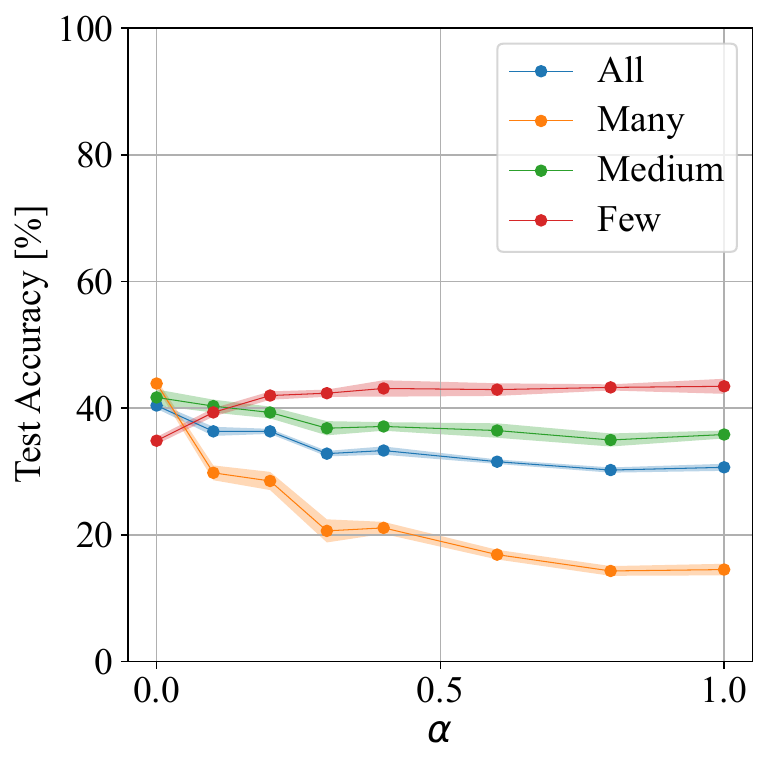} \\
		(a) BalPoE & (b) Head expert & (c) Medium expert & (d) Tail expert
	\end{tabular}
	\caption{Test accuracy on CIFAR-100-LT with IR=100, as a function of the mixup parameter $\alpha$, for (a) BalPoE with three experts, (b) head expert, (c) medium expert, and (d) tail expert.}
	\label{fig:alpha_vs_test_error_ensemble}
\end{figure*}

\paragraph{Results on CIFAR-10-LT.}
Table \ref{tab:sota_cifar10} presents results for CIFAR-10-LT, which includes 10 classes, under different imbalance ratios. By default, we train BalPoE with mixup regularization ($\alpha=0.8$). We observe that our approach promotes a consistent boost in performance under less extreme scenarios, where there are a few classes with arguably enough data. Moreover, we demonstrate that, despite the lower difficulty of this task, BalPoE can still benefit from stronger data augmentation and more extended training, pushing the state-of-the-art on CIFAR-10-LT across several levels of class imbalance.

\begin{table*}[ht]
\caption{Test accuracy (\%) of ResNet32 on CIFAR-10-LT for different imbalance ratios (IR). $\star$: Our reproduced results. $\dagger$: From \cite{xu2021bayias}. $\ddagger$: From \cite{zhang2021test}. $\S$: From \cite{zhou2020bbn}.
}

\label{tab:sota_cifar10}
\renewcommand{\arraystretch}{0.5}
\centering
\begin{tabular}{lccc}
\toprule
& \multicolumn{3}{c}{CIFAR-10-LT}\\
\cmidrule(r){2-4}

Method $\downarrow$ \qquad\qquad\qquad\quad IR $\rightarrow$ & 10 & 50 & 100 \\
\midrule
CE$^\star$ & 87.2{\scriptsize $\pm$0.3} & 77.3{\scriptsize $\pm$0.4} & 71.3{\scriptsize $\pm$0.9} \\ 
LDAM-DRW \cite{cao2019learning} & 88.2 & 81.8$^\dagger$ & 77.1 \\
BS \cite{ren2020balanced} & \textbf{90.9{\scriptsize $\pm$0.4}} & - & 83.1{\scriptsize $\pm$0.4} \\
MiSLAS \cite{zhong2021mislas_mixup} & 90.0 & \underline{85.7} & 82.1 \\
RIDE$^\ddagger$ \cite{wang2020experts} & 89.7  & - & 81.6 \\
ACE \cite{cai2021ace} & - & 84.3 & 81.2 \\
UniMix+Bayias \cite{xu2021bayias} & 89.7 & 84.3 & 82.7 \\
TLC \cite{li2022TLC_trustworthy} & - & - & 80.3 \\
SADE \cite{zhang2021test} & \underline{90.8} & - & \underline{83.8} \\

\textbf{BalPoE} (ours) & 90.2{\scriptsize $\pm$0.2} & \textbf{86.2{\scriptsize $\pm$0.2}} & \textbf{84.2{\scriptsize $\pm$0.3}} \\
\midrule
\textit{Longer training} \\
\midrule
NCL \cite{li2022NCL_nested_collab} & - & 87.3 & 85.5 \\ 
\textbf{BalPoE} (ours) & \textbf{91.9{\scriptsize $\pm$0.1}} & \textbf{88.5{\scriptsize $\pm$0.2}} & \textbf{86.8{\scriptsize $\pm$0.2}} \\
\bottomrule
\end{tabular}

\end{table*}

\newpage

\subsection{Extended calibration comparison}

\newpage

\paragraph{Definition of calibration.} Intuitively, calibration is the measure of how well the model confidence reflects the true probability, i.e. when the model predicts a class with $90\%$ confidence, it should be the correct class in $90\%$ of the cases on average. Formally, a model $h$ is said to be \textit{calibrated} \cite{brocker2009perfect_calibration} if 
\begin{equation}
    \mathbb{P}(Y = y | h(X) = \bm{p} ) = \bm{p}_y \ \ \ \ \ \  \forall \bm{p} \in \Delta, 
    \label{eqn:perfect_calibration}
\end{equation}
where $\Delta = \{p \in [0, 1]^C | \sum_{y \in \mathcal{Y}} p_y = 1\} $ is a (C-1)-dimensional simplex. 
A strictly weaker, but more useful, condition is \textit{argmax calibration} \cite{guo2017calibration}, which requires 
\begin{equation}
    \mathbb{P}(Y = \argmax h(X) | \max h(X) = p ) = p \ \ \ \ \ \  \forall p \in [0, 1].
    \label{eqn:argmax_calibration}
\end{equation}

In practice, we empirically estimate the disagreement between the two sides of \eqref{eqn:argmax_calibration} over a discrete set of samples. Given a dataset $\mathcal{D} = \{x_i, y_i\}^{N}_{i=1}$, denote $\hat{p}_i$ the predicted confidence of sample $x_i$. \cite{guo2017calibration} propose to group predictions into $M$ discrete intervals, and then calculate accuracy and confidence over the respective batch of samples. Let $B_m$ denote the batch of indices in the $m$ interval, we define the average accuracy of $B_m$ as $acc(B_m) = \frac{1}{|B_m|} \sum_{i \in B_m} \mathbf{1}(\hat{y}_i = y_i)$. Similarly, the average confidence of $B_m$ is defined as $conf(B_m) = \frac{1}{|B_m|} \sum_{i \in B_m} \hat{p}_i$. We estimate the Expected Calibration Error (ECE) as a weighted average of the batch's differences between accuracy and confidence, i.e. 
\begin{equation}
    ECE = \sum_{m=1}^{M} \frac{|B_m|}{n} |acc(B_m) - conf(B_m)|,
    \label{eqn:ece_metric}
\end{equation}
where $n$ denotes the number of samples in each equally-spaced interval. Analogously, the Maximum Calibration Error (MCE) describes the maximum difference between accuracy and confidence, i.e.
\begin{equation}
    MCE = \max_{m \in \{1,..,M\}} |acc(B_m) - conf(B_m)|.
    \label{eqn:mce_metric}
\end{equation}

\paragraph{Extended discussion.}  We present reliability diagrams in Figure \ref{fig:reliability} and Figure \ref{fig:reliability_cifar10} for CIFAR-100-LT-100 and CIFAR-10-LT-100, respectively, where we plot the accuracy as a function of the model confidence. Ideally, for samples where the confidence is $C$, the rate at which the prediction is correct should be the same, namely $C$. This is highlighted by the diagonal line in the diagram, which corresponds to a perfectly calibrated model. For CIFAR-100-LT-100, the ECE for a single model trained with ERM is $31.5\%$, which is reduced to $23.1\%$ for BS (equivalent to $\lambda=0$), and further reduced to $16.9\%$ with a BalPoE of 3 experts ($\lambda=\{1, 0, -1\}$). Remarkably, mixup can further improve the calibration of our approach, leading to an ECE of $4.1\%$. We observe similar gains for CIFAR-10-LT-100 in terms of calibration, see  Figure \ref{fig:reliability_cifar10}, and generalization performance, as shown in Table \ref{tab:sota_cifar10_calibration}. We conclude that meeting the calibration assumption is vital for our logit-adjusted expert framework, which we argue explains the large performance gains obtained by using mixup.

\begin{table}[t]
\renewcommand{\arraystretch}{0.7}
\caption{Expected calibration error (ECE), maximum calibration error (MCE), and test accuracy (ACC) on CIFAR-10-LT-100. $\star$: Our reproduced results, where mixup is trained with $\alpha=0.8$. $\dagger$: from \cite{xu2021bayias}. $\ddagger$: our approach trained with ERM.}

\label{tab:sota_cifar10_calibration}
\centering

\resizebox{0.5\textwidth}{!}{

\begin{tabular}{lccc}
\toprule
& \multicolumn{3}{c}{CIFAR-10-LT-100} \\
\cmidrule(r){2-4}
Method $\downarrow$ & ECE $\downarrow$ & MCE $\downarrow$ & ACC $\uparrow$\\
\midrule

CE$^\star$ & 19.1{\scriptsize $\pm$0.9} & 33.9{\scriptsize $\pm$2.0} & 71.3{\scriptsize $\pm$0.9} \\
Bayias \cite{xu2021bayias} & \textbf{11.0} & \textbf{23.7} & 78.7 \\
TLC \cite{li2022TLC_trustworthy} & 13.1 & - & \underline{80.3} \\

\textbf{BalPoE}$^\ddagger$ (ours) & \textbf{11.0{\scriptsize $\pm$0.3}} & \underline{27.7{\scriptsize $\pm$2.7}} & \textbf{80.5{\scriptsize $\pm$0.3}} \\

\midrule

Mixup$^\star$\cite{zhang2017mixup} & \textbf{3.7{\scriptsize $\pm$0.4}} & \textbf{12.9{\scriptsize $\pm$4.9}} & 72.9{\scriptsize $\pm$0.7} \\
Remix$^\dagger$\cite{chou2020remix} & 15.4 & 28.0 & 75.4 \\

MiSLAS \cite{zhong2021mislas_mixup} & \textbf{3.7} & - & 82.1 \\
UniMix+Bayias \cite{xu2021bayias} & 10.2 & 25.5 & \underline{82.7} \\

\textbf{BalPoE} (ours) & \underline{6.3{\scriptsize $\pm$0.7}}  & \underline{15.8{\scriptsize $\pm$4.7}} & \textbf{84.2{\scriptsize $\pm$0.3}} \\
\bottomrule
\end{tabular}

}

\end{table}

\begin{figure*}[!t]
	\centering
	\setlength{\tabcolsep}{0pt}
	\begin{tabular}{ccccc}
	
		\includegraphics[width=0.2\textwidth]{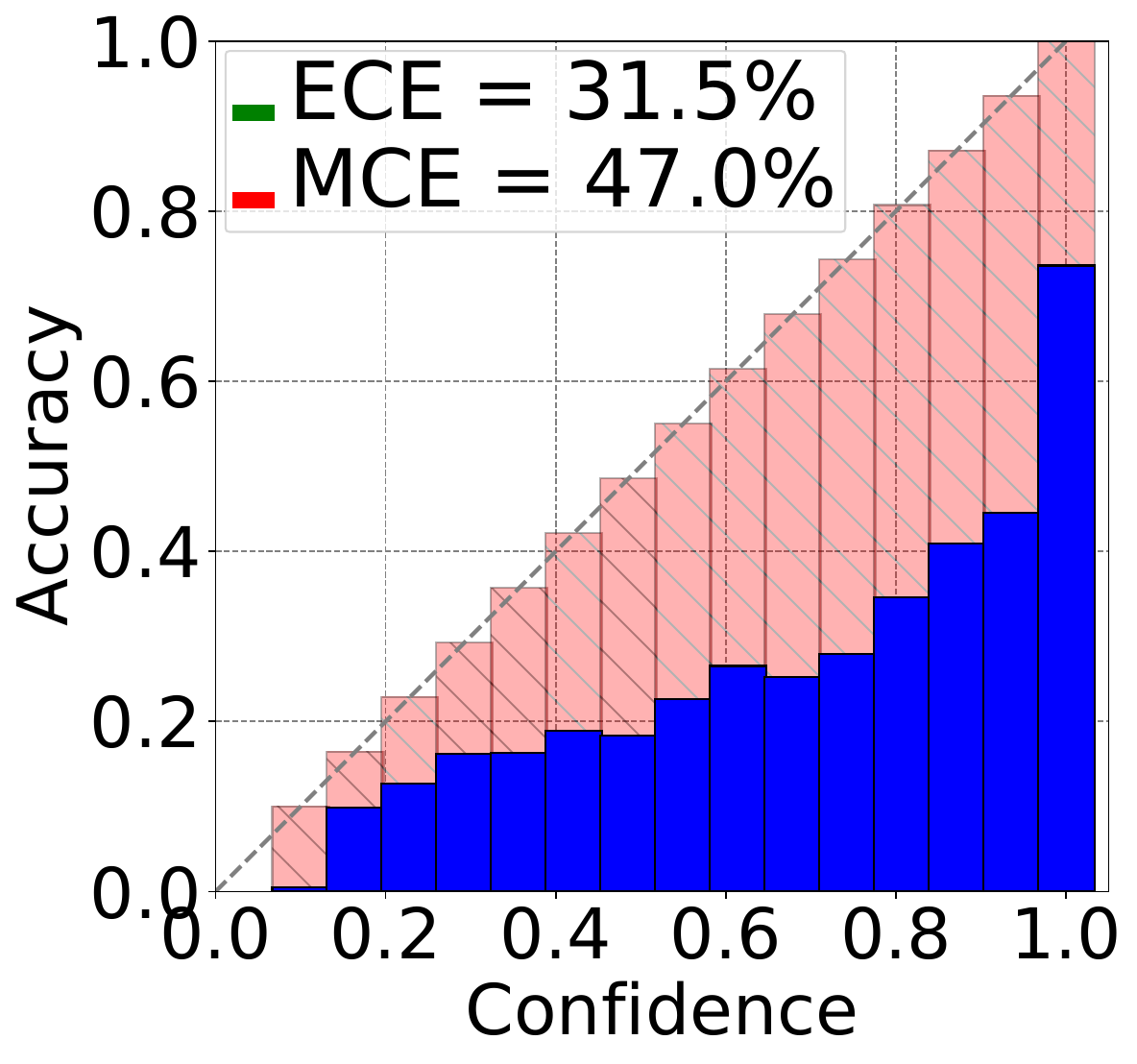} &
		\includegraphics[width=0.2\textwidth]{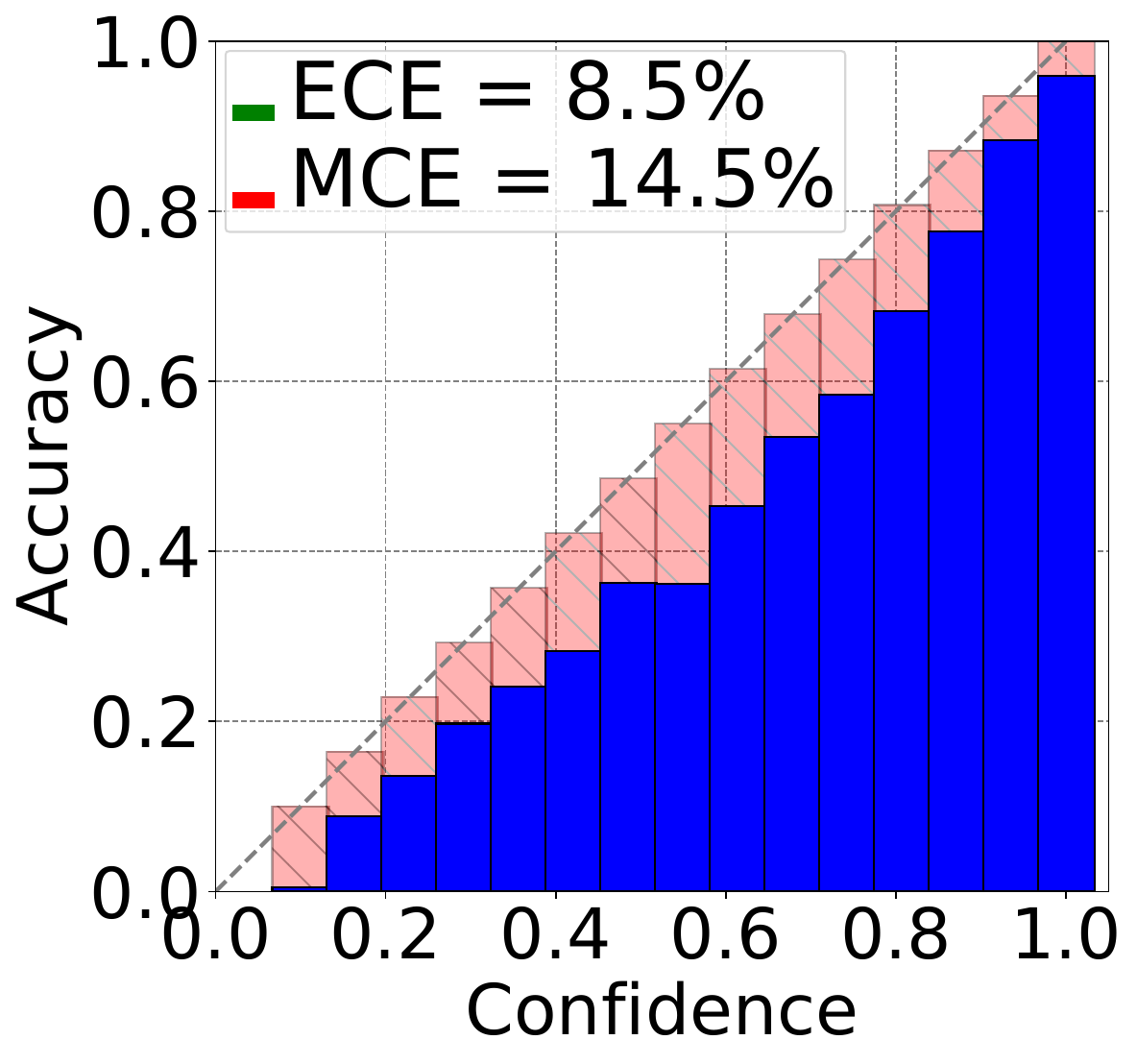} &			
		\includegraphics[width=0.2\textwidth]{figures/cifar100_ir100/reliability_diagr_cifar100_ir100_m=1_tau_avg=1.0_delta=0.0_alpha=0.0.pdf} &
		\includegraphics[width=0.2\textwidth]{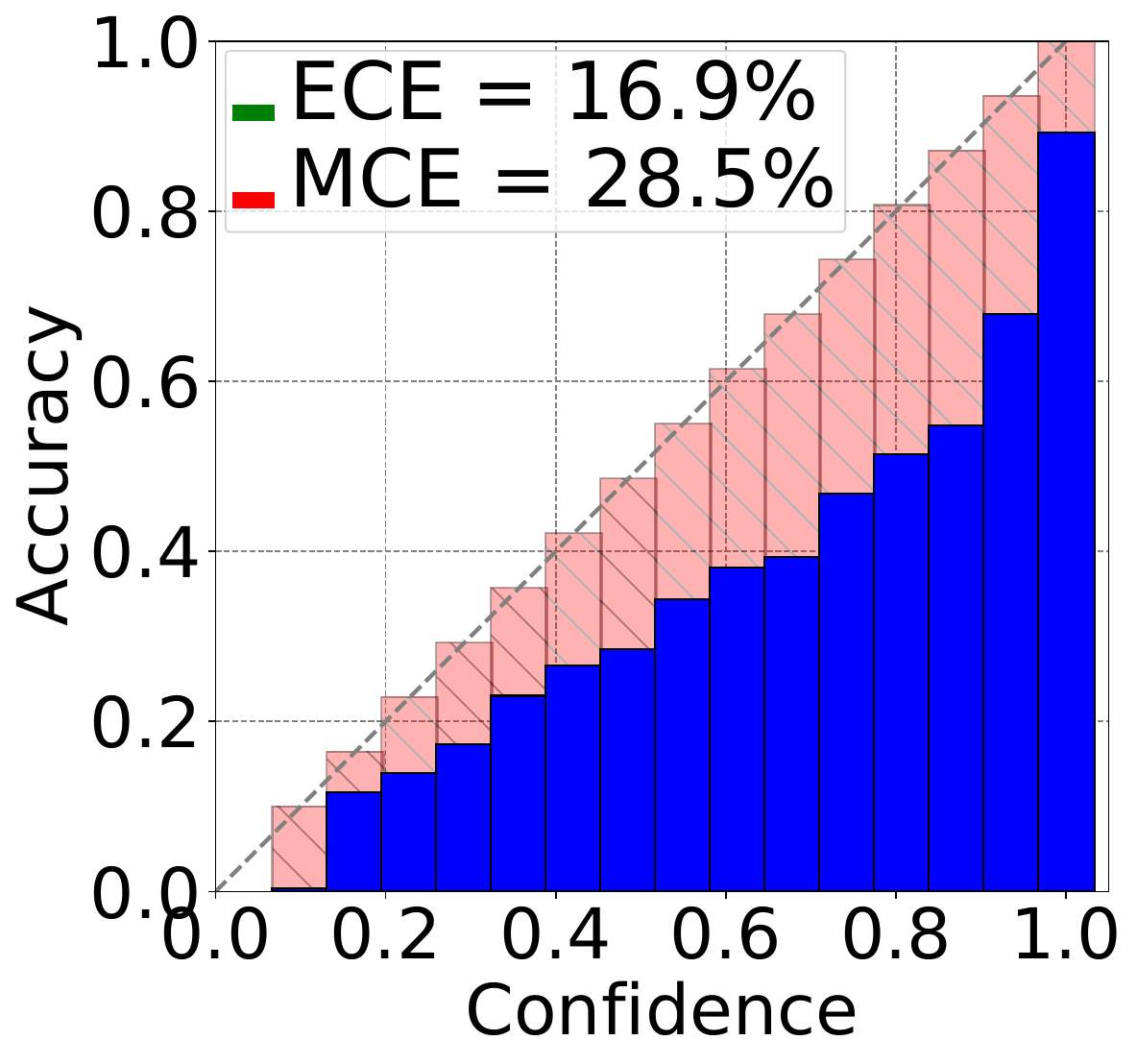} &
		\includegraphics[width=0.2\textwidth]{figures/cifar100_ir100/reliability_diagr_cifar100_ir100_m=3_tau_avg=1.0_delta=1.0_alpha=0.4.pdf} \\

		(a) CE & (b) Mixup & (c) BS & (d) Uncal. BalPoE & (e) BalPoE
		
	\end{tabular}
	\caption{Reliability plots for (a) CE, (b) mixup, (c) BS, (d) uncalibrated BalPoE (trained with ERM) and (d) BalPoE (trained with mixup). Computed over \textbf{CIFAR-100-LT-100} test set.}
	\label{fig:reliability}
\end{figure*}

\begin{figure*}[!t]
	\centering
	\setlength{\tabcolsep}{0pt}
	\begin{tabular}{ccccc}
		\includegraphics[width=0.2\textwidth]{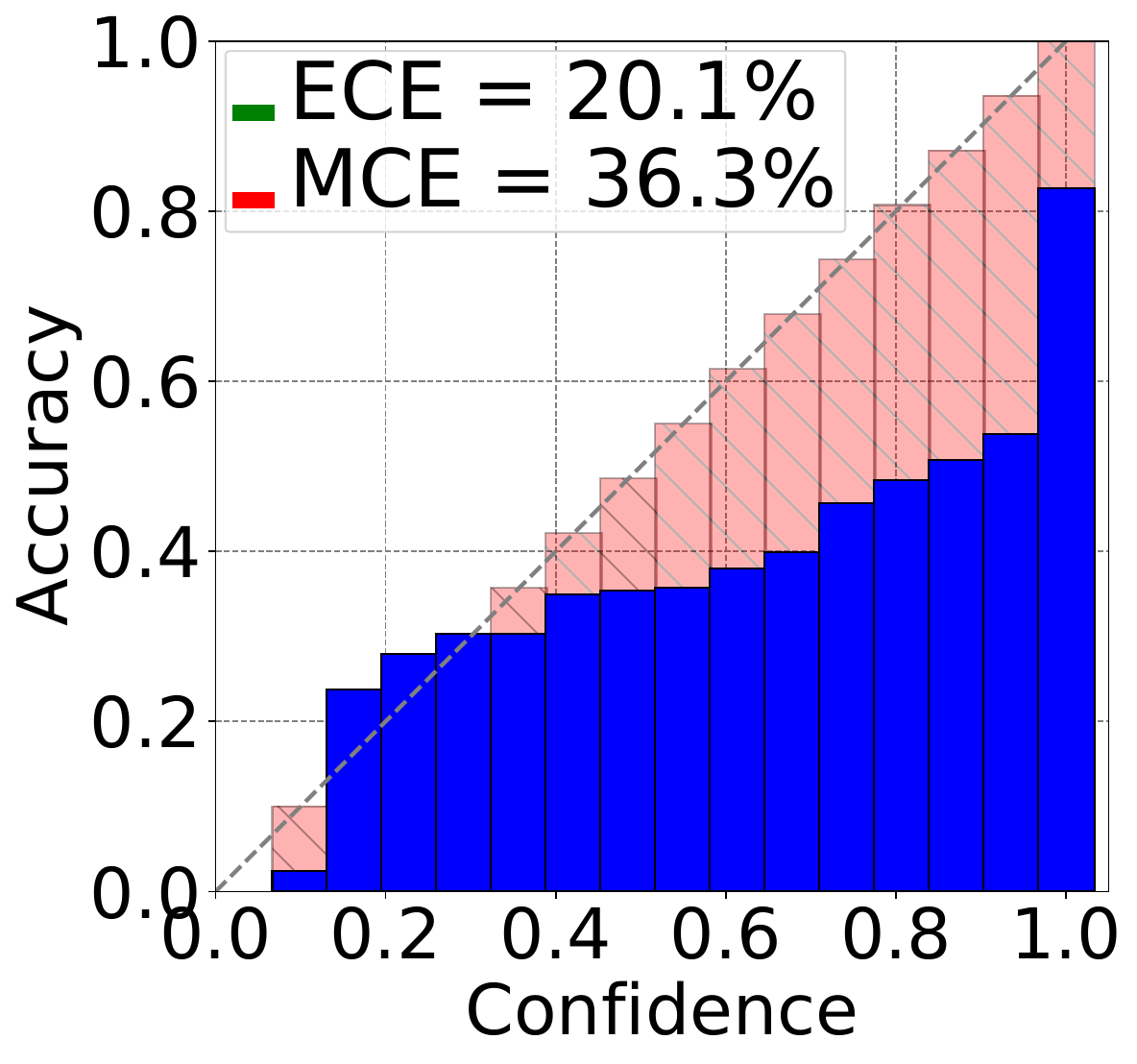} &
		\includegraphics[width=0.2\textwidth]{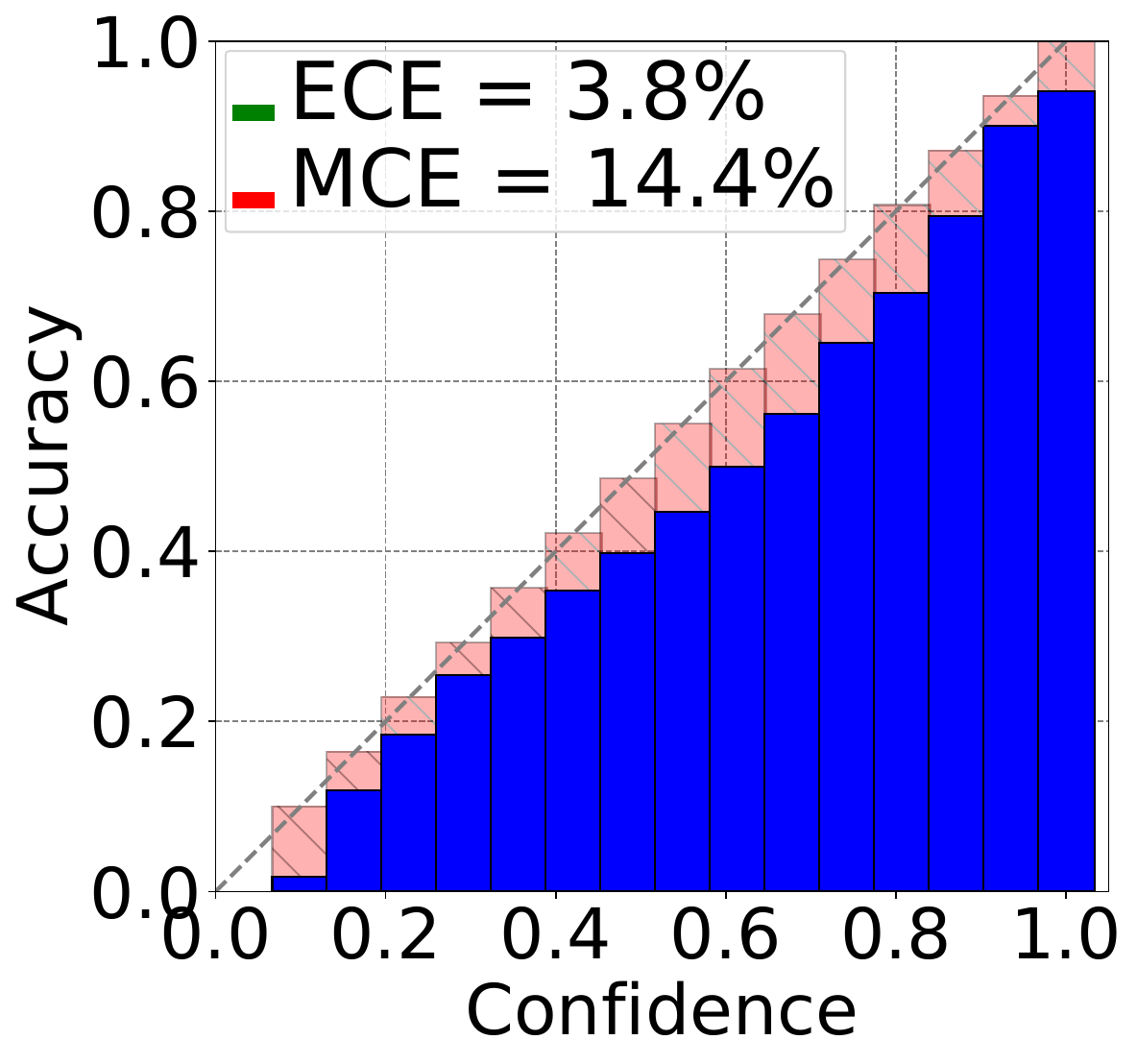} &			
		\includegraphics[width=0.2\textwidth]{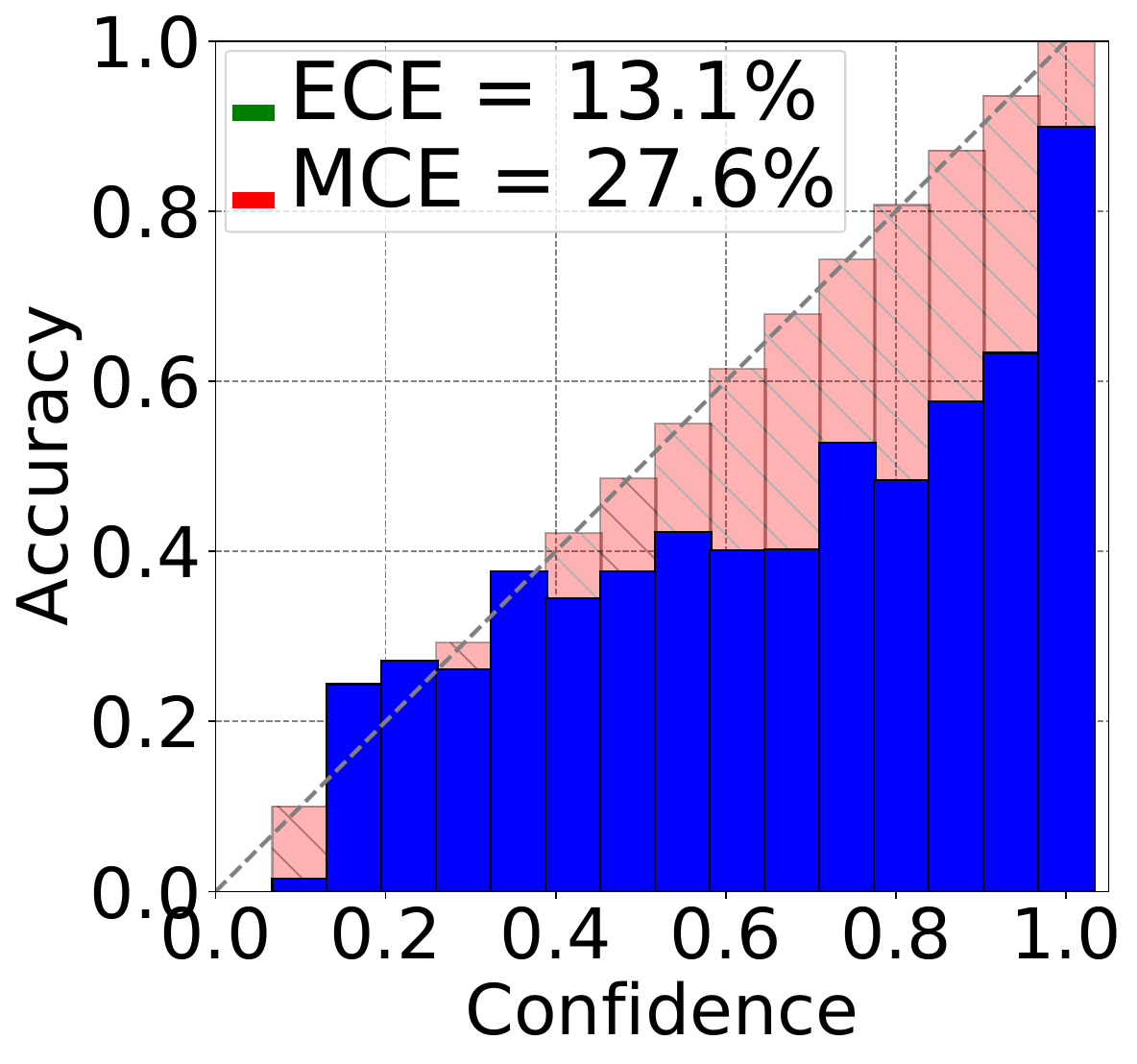} &
		\includegraphics[width=0.2\textwidth]{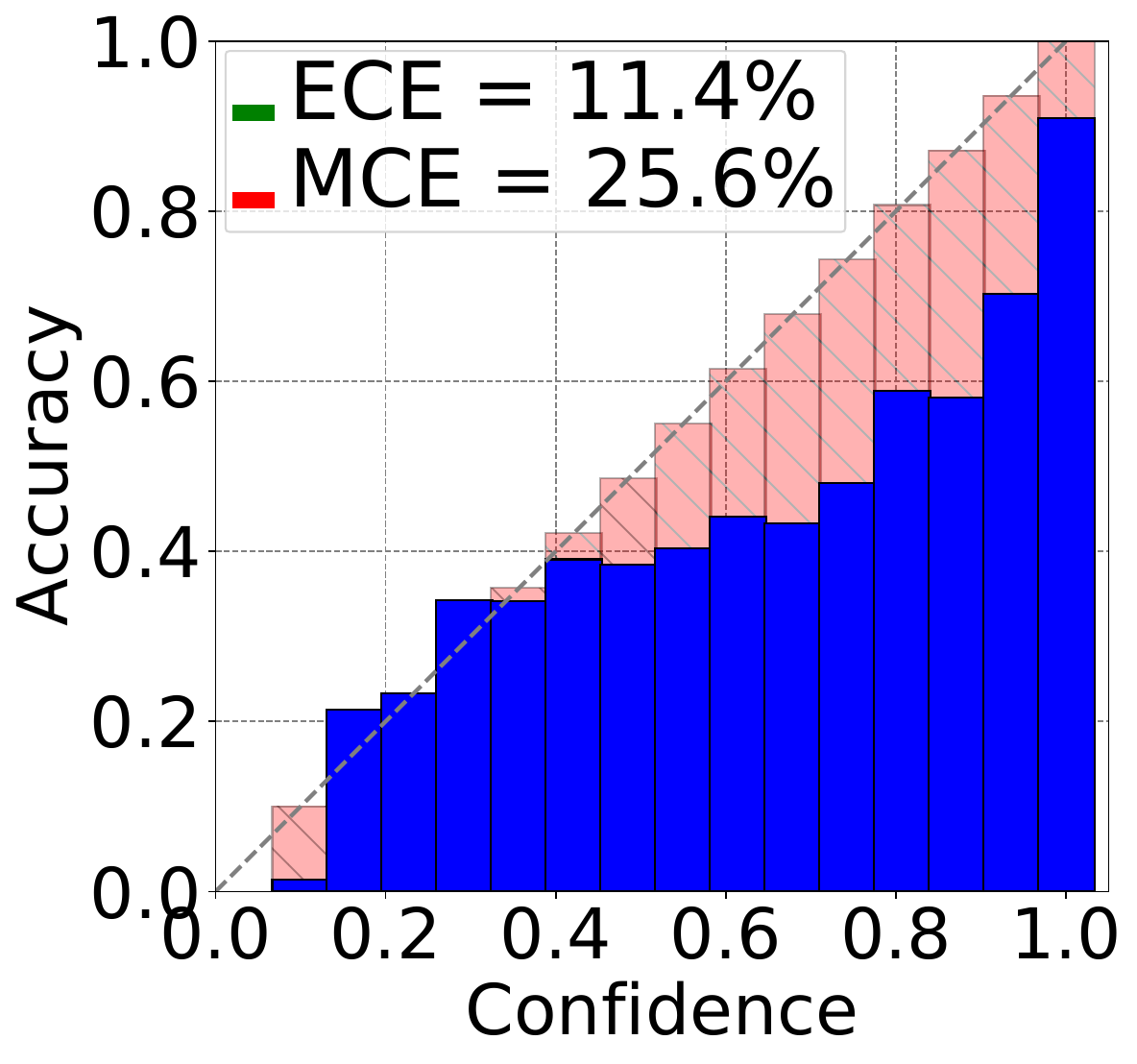} &
		\includegraphics[width=0.2\textwidth]{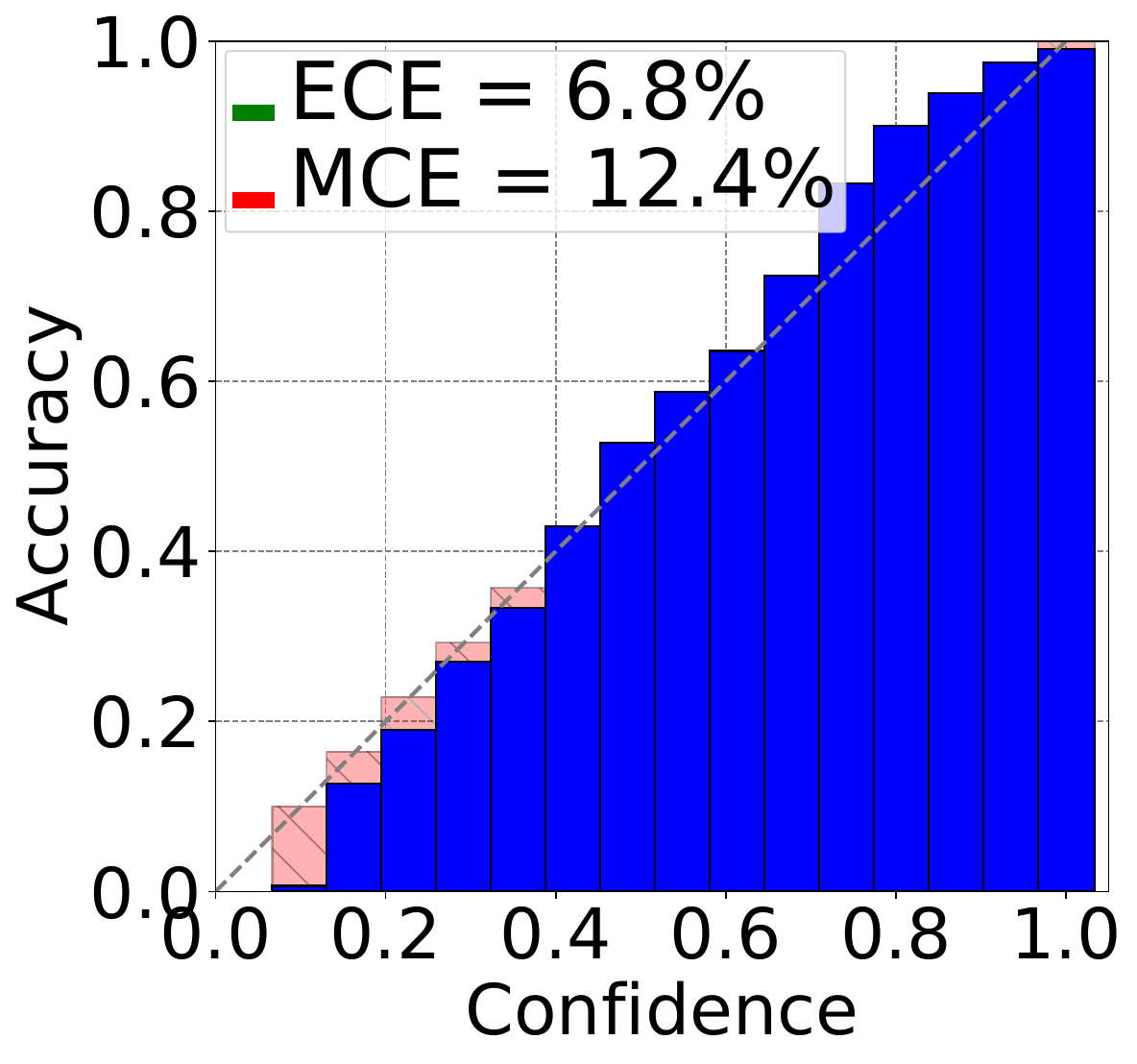} \\
		(a) CE & (b) Mixup & (c) BS & (d) Uncal. BalPoE & (e) BalPoE
	\end{tabular}
	\caption{Reliability plots for (a) CE, (b) mixup, (c) BS, (d) uncalibrated BalPoE (trained with ERM), and (e) BalPoE (trained with mixup). Computed over \textbf{CIFAR-10-LT-100} test set.}
	\label{fig:reliability_cifar10}
\end{figure*}

\comment{

\begin{figure*}
	\centering
	\begin{tabular}{cc}
		\includegraphics[height=0.27\textwidth]{figures/cifar100_ir100/alpha_vs_expected_calib_error_for_taus.pdf} &
		\includegraphics[height=0.27\textwidth]{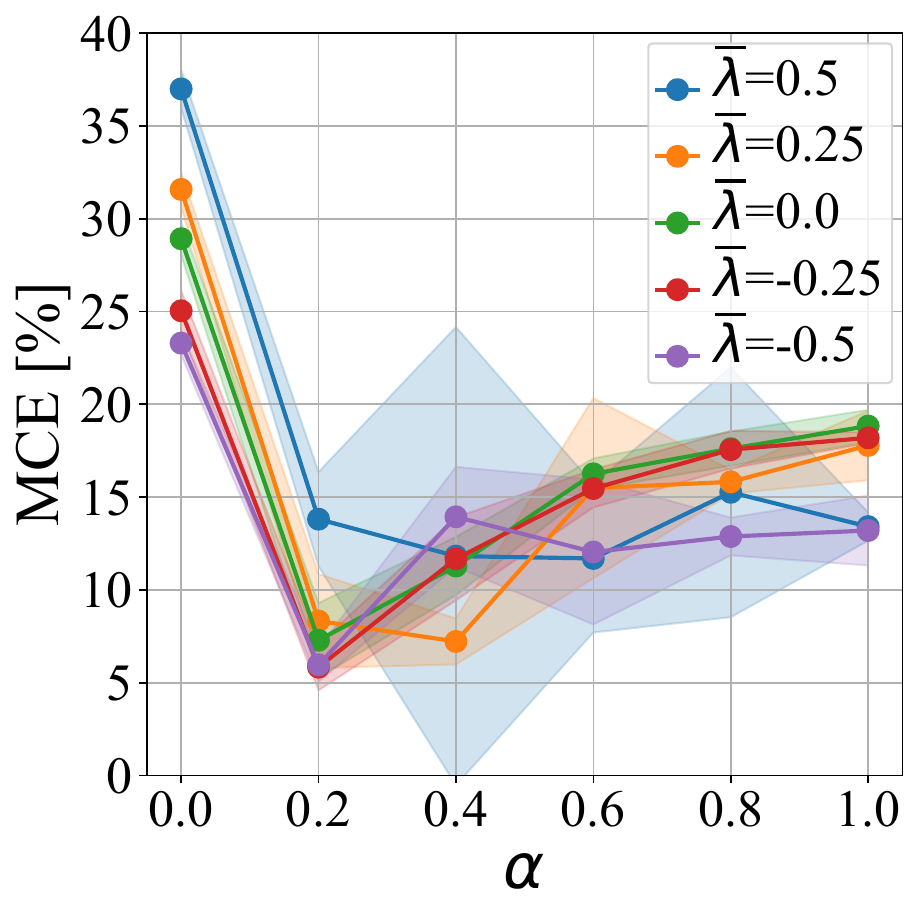} \\
		(a) & (b)
	\end{tabular}
	\caption{From left to right, (a) expected calibration error (ECE) and (b) maximum calibration error (MCE) on CIFAR-100-LT-100 test split.}
	\label{fig:calib_errors_extended}
\end{figure*}

}

\newpage

\subsection{Extended comparison under diverse test distributions}

\paragraph{Definition of shifted long-tailed datasets.} Following \cite{hong2021lade,zhang2021test}, we evaluate our approach under various test class distributions with different imbalance ratios, in order to simulate the diversity of real-world situations. We group these datasets into forward long-tailed distributions, the uniform distribution, and backward long-tailed distributions. For forward distributions, the classes are sorted in decreasing order according to the number of training samples, whereas for backward distributions the class order is flipped. See \cite{hong2021lade,zhang2021test} for a comprehensive description of these benchmarks.

\paragraph{Extended discussion.} In Tables \ref{tab:sota_cifar100-lt100_test_distributions}, \ref{tab:sota_cifar100-lt50_test_distributions} and \ref{tab:sota_cifar100-lt10_test_distributions} we present additional results of multiple shifted target distributions for CIFAR-100-LT with a training imbalance ratio of 100, 50 and 10, respectively. Across different distributions, our approach provides a significantly better \textit{head-tail trade-off} than other expert-based frameworks, outperforming SADE and RIDE by notable margins at forward and backward scenarios, respectively. Remarkably, the benefits of our unbiased ensemble can also be appreciated as more training data becomes available, particularly for IR$=$50 and IR$=$10. Our extensive evaluation further corroborates our early hypothesis: \textit{an ensemble of well-calibrated experts can be a more robust long-tailed classifier than single-expert (often uncalibrated) logit-adjusted approaches}, such as BS and LADE. Tables \ref{tab:sota_imagenet_test_distributions} and \ref{tab:sota_inaturalist_test_distributions} show additional results for ImageNet-LT and iNaturalist datasets, respectively, where we observe the effectiveness of our framework to tackle LT recognition under challenging large-scale datasets.

\begin{table*}[!t]
\renewcommand{\arraystretch}{0.5}
\begin{center}
\caption{Test accuracy (\%) on multiple test distributions for model trained on CIFAR-100-LT-100. $\dagger$: results from \cite{zhang2021test}. \textit{Prior}: test class distribution is used. $*$: Prior implicitly estimated from test data by self-supervised learning.}
\label{tab:sota_cifar100-lt100_test_distributions}
\begin{tabular}{lcccccccccccc}
    \toprule
    & & \multicolumn{11}{c}{CIFAR-100-LT-100} \\
    \midrule
    & & \multicolumn{5}{c}{Forward-LT} & Unif. & \multicolumn{5}{c}{Backward-LT} \\    
    \cmidrule(lr){3-7}
    \cmidrule(lr){8-8}
    \cmidrule(lr){9-13}
     Method & prior $\downarrow$ IR $\rightarrow$ & 50 & 25 & 10 & 5 & 2 & 1 & 2 & 5 & 10 & 25 & 50 \\
    \midrule
    Softmax$^\dagger$ & \ding{55} & 63.3 & 62.0 & 56.2 & 52.5 & 46.4 & 41.4 & 36.5 & 30.5 & 25.8 & 21.7 & 17.5 \\
    BS$^\dagger$ &\ding{55}  & 57.8 & 55.5 & 54.2 & 52.0 & 48.7 & 46.1 & 43.6 & 40.8 & 38.4 & 36.3 & 33.7 \\
    MiSLAS$^\dagger$ & \ding{55} & 58.8 & 57.2 & 55.2 & 53.0 & 49.6 & 46.8 & 43.6 & 40.1 & 37.7 & 33.9 & 32.1 \\
    LADE$^\dagger$ & \ding{55} & 56.0 & 55.5 & 52.8 & 51.0 & 48.0 & 45.6 & 43.2 & 40.0 & 38.3 & 35.5 & 34.0 \\
    
    RIDE$^\dagger$ & \ding{55} & 63.0 & 59.9 & 57.0 & 53.6 & 49.4 & 48.0 & 42.5 & 38.1 & 35.4 & 31.6 & 29.2 \\
    SADE & \ding{55} & 58.4 & 57.0 & 54.4 & 53.1 & 50.1 & 49.4 & 45.2 & 42.6 & 39.7 & 36.7 & 35.0 \\
    
    \textbf{BalPoE} & \ding{55} & \textbf{65.1} & \textbf{63.1} & \textbf{60.8} & \textbf{58.4} & \textbf{54.8} & \textbf{52.0} & \textbf{48.6} & \textbf{44.6} & \textbf{41.8} & \textbf{38.0} & \textbf{36.1} \\

    \midrule
    LADE$^\dagger$ & \checkmark & 62.6 & 60.2 & 55.6 & 52.7 & 48.2 & 45.6 & 43.8 & 41.1 & 41.5 & 40.7 & 41.6 \\
    SADE & * & 65.9 & 62.5 & 58.3 & 54.8 & 51.1 & 49.8 & 46.2 & 44.7 & 43.9 & 42.5 & 42.4 \\
    \textbf{BalPoE} & \checkmark & \textbf{70.3} & \textbf{66.8} & \textbf{62.7} & \textbf{59.3} & \textbf{54.8} & \textbf{52.0} & \textbf{49.2} & \textbf{46.9} & \textbf{46.2} & \textbf{45.4} & \textbf{46.1} \\
    \bottomrule
\end{tabular}
\end{center}
\end{table*}

\begin{table*}[!t]
\renewcommand{\arraystretch}{0.5}
\begin{center}
\caption{Test accuracy (\%) on multiple test distributions for model trained on CIFAR-100-LT-50. $\dagger$: results from \cite{zhang2021test}. \textit{Prior}: test class distribution is used. $*$: Prior implicitly estimated from test data by self-supervised learning.}
\label{tab:sota_cifar100-lt50_test_distributions}
\begin{tabular}{lcccccccccccc}
    \toprule
    & & \multicolumn{11}{c}{CIFAR-100-LT-50} \\
    \midrule
    & & \multicolumn{5}{c}{Forward-LT} & Unif. & \multicolumn{5}{c}{Backward-LT} \\    
    \cmidrule(lr){3-7}
    \cmidrule(lr){8-8}
    \cmidrule(lr){9-13}
     Method & prior $\downarrow$ IR $\rightarrow$ & 50 & 25 & 10 & 5 & 2 & 1 & 2 & 5 & 10 & 25 & 50 \\
    \midrule
    Softmax$^\dagger$ & \ding{55} & 64.8 & 62.7 & 58.5 & 55.0 & 49.9 & 45.6 & 40.9 & 36.2 & 32.1 & 26.6 & 24.6 \\
    BS$^\dagger$ & \ding{55} & 61.6 & 60.2 & 58.4 & 55.9 & 53.7 & 50.9 & 48.5 & 45.7 & 43.9 & 42.5 & 40.6 \\
    MiSLAS$^\dagger$ & \ding{55} & 60.1 & 58.9 & 57.7 & 56.2 & 53.7 & 51.5 & 48.7 & 46.5 & 44.3 & 41.8 & 40.2 \\
    LADE$^\dagger$ & \ding{55} & 61.3 & 60.2 & 56.9 & 54.3 & 52.3 & 50.1 & 47.8 & 45.7 & 44.0 & 41.8 & 40.5 \\
    
    RIDE$^\dagger$ & \ding{55} & 62.2 & 61.0 & 58.8 & 56.4 & 52.9 & 51.7 & 47.1 & 44.0 & 41.4 & 38.7 & 37.1 \\
    SADE & \ding{55} & 59.5 & 58.6 & 56.4 & 54.8 & 53.2 & 53.8 & 50.1 & 48.2 & 46.1 & 44.4 & 43.6 \\
    \textbf{BalPoE} & \ding{55} & \textbf{66.5} & \textbf{64.8} & \textbf{62.8} & \textbf{60.9} & \textbf{58.3} & \textbf{56.3} & \textbf{53.8} & \textbf{51.0} & \textbf{48.9} & \textbf{46.6} & \textbf{45.3} \\

    \midrule
    LADE$^\dagger$ & \checkmark & 65.9 & 62.1 & 58.8 & 56.0 & 52.3 & 50.1 & 48.3 & 45.5 & 46.5 & 46.8 & 47.8 \\
    SADE & * & 67.2 & 64.5 & 61.2 & 58.6 & 55.4 & 53.9 & 51.9 & 50.9 & 51.0 & 51.7 & 52.8 \\
    \textbf{BalPoE} & \checkmark & \textbf{71.1} & \textbf{68.3} & \textbf{64.8} & \textbf{61.8} & \textbf{58.2} & \textbf{56.3} & \textbf{54.4} & \textbf{53.4} & \textbf{53.4} & \textbf{53.8} & \textbf{55.4} \\
    \bottomrule
\end{tabular}
\end{center}
\end{table*}

\begin{table*}[t]
\renewcommand{\arraystretch}{0.5}
\begin{center}
\caption{Test accuracy (\%) on multiple test distributions for model trained on CIFAR-100-LT-10. $\dagger$: results from \cite{zhang2021test}. \textit{Prior}: test class distribution is used. $*$: Prior implicitly estimated from test data by self-supervised learning.}
\label{tab:sota_cifar100-lt10_test_distributions}
\begin{tabular}{lcccccccccccc}
    \toprule
    & & \multicolumn{11}{c}{CIFAR-100-LT-10} \\
    \midrule
    & & \multicolumn{5}{c}{Forward-LT} & Unif. & \multicolumn{5}{c}{Backward-LT} \\    
    \cmidrule(lr){3-7}
    \cmidrule(lr){8-8}
    \cmidrule(lr){9-13}
     Method & prior $\downarrow$ IR $\rightarrow$ & 50 & 25 & 10 & 5 & 2 & 1 & 2 & 5 & 10 & 25 & 50 \\
    \midrule
    Softmax$^\dagger$ & \ding{55} & \textbf{72.0} & \textbf{69.6} & 66.4 & 65.0 & 61.2 & 59.1 & 56.3 & 53.5 & 50.5 & 48.7 & 46.5 \\
    BS$^\dagger$ & \ding{55} & 65.9 & 64.9 & 64.1 & 63.4 & 61.8 & 61.0 & 60.0 & 58.2 & 57.5 & 56.2 & 55.1 \\
    MiSLAS$^\dagger$ & \ding{55} & 67.0 & 66.1 & 65.5 & 64.4 & 63.2 & 62.5 & 61.2 & 60.4 & 59.3 & 58.5 & 57.7 \\
    LADE$^\dagger$ & \ding{55} & 67.5 & 65.8 & 65.8 & 64.4 & 62.7 & 61.6 & 60.5 & 58.8 & 58.3 & 57.4 & 57.7 \\
    
    RIDE$^\dagger$ & \ding{55} & 67.1 & 65.3 & 63.6 & 62.1 & 60.9 & 61.8 & 58.4 & 56.8 & 55.3 & 54.9 & 53.4 \\
    SADE & \ding{55} & 66.3 & 64.5 & 64.1 & 62.7 & 61.6 & 63.6 & 60.2 & 59.7 & 59.8 & 58.7 & 58.6 \\
    \textbf{BalPoE} & \ding{55} & \underline{69.1} & \underline{68.2} & \textbf{67.4} & \textbf{66.8} & \textbf{65.7} & \textbf{65.1} & \textbf{63.8} & \textbf{63.0} & \textbf{62.3} & \textbf{61.8} & \textbf{61.3} \\
    
    \midrule
    LADE$^\dagger$ & \checkmark & 71.2 & 69.3 & 67.1 & 64.6 & 62.4 & 61.6 & 60.4 & 61.4 & 61.5 & 62.7 & 64.8 \\
    SADE & * & 71.2 & 69.4 & 67.6 & 66.3 & 64.4 & 63.6 & 62.9 & 62.4 & 61.7 & 62.1 & 63.0 \\
    \textbf{BalPoE} & \checkmark & \textbf{74.9} & \textbf{72.4} & \textbf{70.0} & \textbf{68.1} & \textbf{66.0} & \textbf{65.1} & \textbf{64.1} & \textbf{64.3} & \textbf{65.0} & \textbf{66.3} & \textbf{67.8} \\
    \bottomrule
\end{tabular}
\end{center}
\end{table*}

\begin{table*}[t]
\renewcommand{\arraystretch}{0.5}
\begin{center}
\caption{Test accuracy (\%) on multiple test distributions for ResNeXt50 trained on Imagenet-LT. $\dagger$: results from \cite{zhang2021test}. \textit{Prior}: test class distribution is used. $*$: Prior implicitly estimated from test data by self-supervised learning.}
\label{tab:sota_imagenet_test_distributions}
\begin{tabular}{lcccccccccccc}
    \toprule
    & & \multicolumn{11}{c}{Imagenet-LT} \\
    \midrule
    & & \multicolumn{5}{c}{Forward-LT} & Unif. & \multicolumn{5}{c}{Backward-LT} \\    
    \cmidrule(lr){3-7}
    \cmidrule(lr){8-8}
    \cmidrule(lr){9-13}
    Method & prior $\downarrow$ IR $\rightarrow$ & 50 & 25 & 10 & 5 & 2 & 1 & 2 & 5 & 10 & 25 & 50 \\
    \midrule
    Softmax$^\dagger$ & \ding{55} & 66.1 & 63.8 & 60.3 & 56.6 & 52.0 & 48.0 & 43.9 & 38.6 & 34.9 & 30.9 & 27.6 \\    
    BS$^\dagger$ & \ding{55} & 63.2 & 61.9 & 59.5 & 57.2 & 54.4 & 52.3 & 50.0 & 47.0 & 45.0 & 42.3 & 40.8 \\
    MiSLAS$^\dagger$ & \ding{55} & 61.6 & 60.4 & 58.0& 56.3 & 53.7 & 51.4 & 49.2 & 46.1 & 44.0 & 41.5 & 39.5 \\
    LADE$^\dagger$ & \ding{55} & 63.4 & 62.1 & 59.9 & 57.4 & 54.6 & 52.3 & 49.9 & 46.8 & 44.9 & 42.7 & 40.7 \\

    RIDE$^\dagger$ & \ding{55} & \textbf{67.6} & \textbf{66.3} & 64.0 & 61.7 & 58.9 & 56.3 & 54.0 & 51.0 & 48.7 & 46.2 & 44.0 \\
    
    SADE & \ding{55} & 65.5 & 64.4 & 63.6 & 62.0 & 60.0 & 58.8 & 56.8 & 54.7 & 53.1 & 51.1 & 49.8 \\
    
    \textbf{BalPoE} & \ding{55} & \textbf{67.6} & \textbf{66.3} & \textbf{65.2} & \textbf{63.3} & \textbf{61.5} & \textbf{59.8} & \textbf{58.1} & \textbf{55.7} & \textbf{54.3} & \textbf{52.2} & \textbf{50.8} \\
    
    \midrule    
    
    LADE$^\dagger$ & \checkmark & 65.8  & 63.8  & 60.6  & 57.5  & 54.5  & 52.3  & 50.4  & 48.8  & 48.6  & 49.0 & 49.2 \\    
    
    SADE & * & 69.4 & 67.4 & 65.4 & 63.0 & 60.6 & 58.8 & 57.1 & 55.5 & 54.5 & 53.7 & 53.1  \\
    
    \textbf{BalPoE} & \checkmark & \textbf{72.5} & \textbf{70.2} & \textbf{67.3} & \textbf{64.6} & \textbf{61.8} & \textbf{59.8} & \textbf{58.3} & \textbf{57.2} & \textbf{56.6} & \textbf{56.6} & \textbf{56.9} \\
    \bottomrule
\end{tabular}
\end{center}
\end{table*}

\begin{table*}[!t]
\renewcommand{\arraystretch}{0.5}
\begin{center}
\caption{Test accuracy (\%) on multiple test distributions for ResNet50 trained on iNaturalist-2018. $\dagger$: results from \cite{zhang2021test}. \textit{Prior}: test class distribution is used. $*$: Prior implicitly estimated from test data by self-supervised learning.}
\label{tab:sota_inaturalist_test_distributions}
\begin{tabular}{lcccccccccccc}
    \toprule
    & & \multicolumn{5}{c}{Inaturalist} \\
    \midrule
    & & \multicolumn{2}{c}{Forward-LT} & Unif. & \multicolumn{2}{c}{Backward-LT} \\    
    \cmidrule(lr){3-4}
    \cmidrule(lr){5-5}
    \cmidrule(lr){6-7}
    Method & prior $\downarrow$ IR $\rightarrow$ & 3 & 2 & 1 & 2 & 3 \\
    \midrule
    Softmax$^\dagger$ & \ding{55} & 65.4 & 65.5 & 64.7 & 64.0 & 63.4 \\
    BS$^\dagger$ & \ding{55} & 70.3 & 70.5 & 70.6 & 70.6 & 70.8 \\
    MiSLAS$^\dagger$ & \ding{55} & 70.8 & 70.8 & 70.7 & 70.7 & 70.2 \\
    LADE$^\dagger$ & \ding{55} & 68.4 & 69.0 & 69.3 & 69.6 & 69.5 \\

    RIDE$^\dagger$ & \ding{55} & 71.5 & 71.9 & 71.8 & 71.9 & 71.8 \\
    SADE & \ding{55} & - & 72.4 & 72.9 & 73.1 & - \\
    
    \textbf{BalPoE} & \ding{55} & \textbf{74.3} & \textbf{75.0} & \textbf{75.0} & \textbf{75.1} & \textbf{74.7} \\
    
    \midrule    
    
    LADE$^\dagger$ & \checkmark & - & 69.1 & 69.3 & 70.2 & - \\    
    
    SADE & * & 72.3 & 72.5 & 72.9 & 73.5 & 73.3  \\
    
    \textbf{BalPoE} & \checkmark & \textbf{74.7} & \textbf{75.4} & \textbf{75.0} & \textbf{75.6} & \textbf{75.3} \\
    \bottomrule
\end{tabular}
\end{center}
\end{table*}